\documentclass[lettersize,journal]{IEEEtran}
\usepackage{amsmath,amsfonts}
\usepackage{algorithmic}
\usepackage{algorithm}
\usepackage{array}
\usepackage[caption=false,font=normalsize,labelfont=sf,textfont=sf]{subfig}
\usepackage{textcomp}
\usepackage{stfloats}
\usepackage{amsmath,amssymb}
\usepackage{mathrsfs}
\usepackage{url}
\usepackage{verbatim}
\usepackage{graphicx}
\usepackage{cite}
\usepackage[mathlines]{lineno}
\usepackage{amsthm}
\usepackage{bbm}
\usepackage{bm}
\usepackage{diagbox}
\usepackage{booktabs}
\newtheorem{assumption}{Assumption}
\usepackage[dvipsnames]{xcolor}
\newtheorem{theorem}{Theorem}

\newtheorem{lemma}{Lemma}

\newtheorem{corollary}{Corollary}
\newtheorem{remark}{Remark}

\DeclareMathOperator{\rank}{rank}

\DeclareMathOperator{\ii}{\bm{\mathrm{i}}}

\usepackage[autostyle=false, style=english]{csquotes}
\MakeOuterQuote{"}
\begin{document}
\title{Quantized Low-Rank Multivariate Regression\\with Random Dithering}

\author{Junren Chen\thanks{J. Chen and Y. Wang are with Department of 
		Mathematics, The University of Hong Kong. J. Chen and Y. Wang contributed equally to this work. J. Chen and Y. Wang were supported by Hong Kong PhD Fellowship from Hong Kong Research Grant Council (HKRGC). 
  M. K. Ng is with Department of Mathematics, Hong Kong Baptist University. M. K. Ng was partially supported by the HKRGC GRF 
17201020, 17300021, CRF C7004-21GF and Joint NSFC-RGC N-HKU76921. (e-mails: \texttt{chenjr58@connect.hku.hk}; \texttt{u3007895@connect.hku.hk}; \texttt{michael-ng@hkbu.edu.hk})  {\it (Corresponding authors:  Junren Chen and Yueqi Wang.)}}, Yueqi Wang and Michael K. Ng,~\IEEEmembership{Senior Member}
 }
\markboth{ACCEPTED AT IEEE Transactions on Signal Processing}%
{Shell \MakeLowercase{\textit{et al.}}: A Sample Article Using IEEEtran.cls for IEEE Journals}


\maketitle

\begin{abstract}
Low-rank multivariate regression (LRMR) is an important statistical learning model that combines highly 
correlated tasks as a multiresponse regression problem  with low-rank priori on the coefficient matrix.  In this paper, we study quantized LRMR, a practical setting where the responses and/or the covariates are discretized to finite precision.  We focus on the estimation of the underlying coefficient matrix. To make consistent estimator that could achieve arbitrarily small error possible, we employ uniform quantization with random dithering, i.e., we 
add appropriate random noise to the data before quantization. Specifically,   uniform dither and triangular dither are used for responses and covariates, respectively. Based on the quantized data, we propose the  constrained Lasso and regularized Lasso estimators, and derive the non-asymptotic error bounds. With the aid of dithering, the estimators achieve minimax optimal rate, while quantization only slightly worsens the multiplicative factor in the error rate. 
Moreover, we extend our results to  a low-rank  regression model with 
matrix responses. We corroborate and demonstrate our theoretical results via   simulations on synthetic data, image restoration, as well as a real data application.
\end{abstract}

\begin{IEEEkeywords}
multiresponse regression, quantization, M-estimator, low-rankness, dithering. 
\end{IEEEkeywords}
\section{Introduction}\label{sec1}
Quantization is the process of mapping continuous input to a discrete form (e.g. a finite dictionary or a finite number of bits) \cite{gray1998quantization}. Quantization of signals or data recently has received considerable attention in the communities of  signal processing, statistics and machine learning. In some signal processing problems, power consumption, manufacturing cost and chip area of analog-to-digital devices grow exponentially with their resolution \cite{kipnis2018fundamental}. In this situation, it is  infeasible to use high-precision data or signals, and quantization with relatively low resolution is preferable, \textcolor{black}{e.g., see the distributed machine learning system described in \cite{danaee2022distributed}}. 
Besides, in modern machine learning problems   extremely huge datasets and highly complex models are ubiquitous,
which often lead to distributed   learning systems \cite{konevcny2016federated}, i.e., a setting involving repeatedly communication among multiple compute nodes that are oftentimes GPUs linked processors within a single machine or even multiple machines.    When the participating workers are typically large in
number and have slow or unstable internet connections (e.g., low-power or low-bandwidth device such as a mobile device),   the communication cost would become prohibitive \cite{konevcny2016federated,reisizadeh2020fedpaq}, and recent works have studied how to send a small number of bits by  quantization to overcome the bottleneck \cite{jacob2018quantization, hubara2017quantized, zhou2016dorefa,reisizadeh2020fedpaq,shlezinger2020uveqfed,koloskova2019decentralized,aysal2008distributed}. \textcolor{black}{More specifically,   working with low-precision training data has proven useful in reducing computation cost when training linear model, as shown by the experimental results in \cite{zhang2017zipml}. Additionally, while sending the quantized gradient is the mainstream in machine learning, it may be inefficient in distributed learning with a huge number of parameters to learn; in this case, transmitting some important quantized data samples could provably reduce the communication cost \cite{hanna2021quantization}. Thus, it is of particular interest to theoretically investigate the interplay between parameter learning and data quantization in some   fundamental statistical learning or estimation problems, e.g., \cite{dirksen2021covariance,chen2022quantizing,chen2023high}.}

Departing momentarily from quantization, 
low-rank multivariate regression (LRMR),     also known  as multi-task learning and reduced-rank regression \cite{argyriou2008convex,caruana1998multitask,reinsel2023multivariate}, is undoubtedly a widely used statistical machine learning model. For clarity we first provide its mathematical formulation: \begin{equation}
    \label{intro}
    \bm{y}_k = \bm{\Theta}_0^\top\bm{x}_k+\bm{\epsilon}_k,~k=1,...,n,
\end{equation}
and the main goal is to learn the  underlying parameter $\bm{\Theta}_0\in \mathbb{R}^{d_1\times d_2}$ from the covariate-response pairs $(\bm{x}_k,\bm{y}_k)\in \mathbb{R}^{d_1}\times \mathbb{R}^{d_2}$. Compared to the canonical regression problem with scalar response (e.g., linear regression), the core spirit of LRMR is to combine and jointly solve $d_2$ highly correlated tasks. In particular, the coefficient vectors of the $d_2$ tasks are merged into $\bm{\Theta}_0$ in (\ref{intro}), and the low-rankness of $\bm{\Theta}_0$ is often assumed to exploit the "intrinsic relatedness" of the $d_2$ learning problems (e.g., \cite{negahban2011estimation,fan2021shrinkage,giraud2011low}).
This model can capture many natural phenomena and hence has a broad range of applications. For example, in genomics study \cite{bunea2010adaptive},  the gene expression profiles ($\bm{y}_k$) and the genetic markers ($\bm{x}_k$) can be approximately associated through only a few linear combinations of highly-correlated genetic markers. Therefore, recovering
a low-rank  and sometimes also sparse  coefficient matrix holds the key to reveal such connections between the responses and predictors. 
In addition, in the study of functional magnetic resonance imaging (fMRI) \cite{harrison2003multivariate}, each voxel within the brain is represented by a time series of neurophysiological activity. Combining with the multivariate  voxel-based time series, researchers use a linear model to describe the underlying large-scale network connectivities among functionally specialized regions in the brain. A practical way is to use a suitable matrix to identify these complex interconnections in the brain; while aiming at modelling the connections via only a small subset of the given data, one often imposes appropriate structures (e.g., low-rankness, sparsity) on the coefficient matrix. Besides, other applications include analysis of electroencephalography (EEG) data decoding \cite{anderson1998multivariate}, neural response modeling \cite{brown2004multiple}, analysis of financial data \cite{reinsel2023multivariate}, chemometrics, psychometrics and econometrics \cite{yuan2007dimension}, to name just a few.

Note that in real applications, data are inevitably quantized to bit streams for the purpose of storage, processing and transmission. Also, the problem of LRMR can possibly arise in a distributed learning regime where quantization is necessary to render lower communication cost. Thus, a natural question is to study quantized LRMR, i.e., LRMR where data are quantized by some mechanism, and one can only access the quantized data for subsequent learning procedure. In signal and image processing,    the most natural quantization method   is arguably the uniform quantizer that discretizes data in a uniform manner  \cite{gray1993dithered,gray1998quantization,lim1990two}. More precisely, given quantization level $\delta$, a real scalar $a$ is quantized to $\mathcal{Q}_\delta(a):=\delta\big(\lfloor\frac{a}{\delta}\rfloor +\frac{1}{2}\big)\in \delta\cdot(\mathbb{Z}+\frac{1}{2})$. For instance, under $\mathcal{Q}_{1}(.)$, the data would be discretized to $\{...,-\frac{3}{2},-\frac{1}{2},\frac{1}{2},\frac{3}{2},...\}$, depending on which cell they belong to.   Unfortunately, directly applying the uniform quantizer $\mathcal{Q}_\delta(.)$ to  LRMR will limit our ability to learn the desired $\bm{\Theta}_0$ --- specifically, consistent estimator\footnote{In regression problems, an estimator is consistent if its estimation error vanishes when sample size tends to infinity.} is in general impossible from $(\bm{x}_k,\mathcal{Q}_\delta(\bm{y}_k))$. This is true even if we only have one task ($d_2=1$): consider a problem with binary features (i.e., $\bm{x}_k \in \{-1,1\}^{d_1}$) and without noise; if we quantize $\bm{y}_k$ to $\mathcal{Q}_1(\bm{y}_k)$, then we can never distinguish $\bm{\Theta}_{01}=[0.5,0.4,0,...,0]^\top$ and $\bm{\Theta}_{02}=[0.5,-0.25,0.2,...,0]^\top$ because $\mathcal{Q}_{1}(\bm{\Theta}_{01}^\top\bm{x})=\mathcal{Q}_{1}(\bm{\Theta}_{02}^\top\bm{x})$ holds for any   $\bm{x}\in \{-1,1\}^{d_1}$.  Later,   this issue will be complemented by the numerical result in Figure \ref{fig4}.

To address the issue, in this paper, we study LRMR under dithered quantization that involves random dithering --- a process that adds random noise to the signal before quantization.  The benefit of dithering for image or speech signals was empirically observed quite early \cite{roberts1962picture,limb1969design,jayant1972application}, while the  theoretical results for quantization error/noise were established in \cite{schuchman1964dither}, see also a cleaner proof provided by \cite{gray1993dithered}. In a nutshell, the benefit of dithering is to whiten the quantization noise. Even more surprisingly, the quantization errors follow i.i.d. uniform distribution (Lemma \ref{lem1}(a)). \textcolor{black}{While we focus on the dithered uniform quantizer, interested readers may consult \cite{widrow2008quantization} for an extensive treatment of quantization noise under various quantizers.}

We deal with the quantization of both response and covariate.
We propose to use uniform dither for $\bm{y}_k$, triangular dither for $\bm{x}_k$ (see precise definition later), and then apply the uniform quantizer. Note that the quantization method is memoryless and thus well suited to hardware implementation.  Our main contributions are as follows: 
\begin{itemize}
    \item Based on the quantized data, we develop an empirical $\ell_2$ loss, which coupled with either nuclear norm constraint or regularization leads to Lasso estimators. We establish minimax optimal non-asymptotic error bounds  for the estimators in the cases of "partial quantization" (i.e., only quantize $\bm{y}_k$) and "complete quantization" (i.e., quantize both $\bm{x}_k$, $\bm{y}_k$). The bounds also characterize how quantization resolution affects the estimation error.  
    \item We show that our quantization method is also applicable to a low-rank linear regression model with matrix response recently studied in \cite{kong2019l2rm}. Our Lasso estimators based on quantized data could still achieve error rate comparable to the full-data regime in \cite{kong2019l2rm}.
\end{itemize}
\subsection{Related works}\label{review}
 There has been rapidly growing literature on quantized compressed sensing   \cite{thrampoulidis2020generalized,dirksen2019quantized,dirksen2021non,plan2012robust,boufounos20081,chen2023high,chen2022quantizing,sun2022quantized,xu2020quantized}, quantized matrix completion \cite{davenport20141,cai2013max,bhaskar2016probabilistic,klopp2015adaptive,chen2023high,chen2022quantizing}, and more recently quantized covariance estimation \cite{dirksen2021covariance,chen2023high,chen2022quantizing,dirksen2023tuning,chen2023parameter}, but we are not aware of any earlier work on quantized LRMR (or more generally put, quantized multiresponse regression). Closest to this paper are prior developments on compressed sensing (CS) under dithered uniform quantization \cite{xu2020quantized,sun2022quantized,chen2022quantizing,thrampoulidis2020generalized}, which we briefly review here. Recall that the (noiseless) CS problem is to recover a structured (e.g., sparse/low-rank) signal $\bm{\theta}_0\in \mathbb{R}^d$ from the data of $(\bm{x}_k,y_k:=\bm{x}_k^\top\bm{\theta}_0)_{k=1}^n$, where $\bm{x}_k$ is the sensing vector, $y_k$ is the measurement, and the high-dimensional setting $n\ll d$ is of primary interest. It was shown that, while quantizing $y_k$ via $\mathcal{Q}_\delta(.)$ with uniform dithering, recovery with near optimal error rate can still be achieved by  constrained Lasso \cite{thrampoulidis2020generalized} or the {\it Projected Back Projection} (PBP) estimator \cite{xu2020quantized}. In \cite{sun2022quantized}, Sun et al. extended \cite{thrampoulidis2020generalized} to corrupted sensing that aims at separating signal and corruption. While \cite{xu2020quantized,sun2022quantized,thrampoulidis2020generalized} only considered the quantization of $y_k$, a recent work \cite{chen2022quantizing} developed the quantization method for $\bm{x}_k$, i.e., via the same dithered uniform quantizer but with uniform dither substituted with triangular dither.

 Although we adopt a similar dithered quantization scheme (specifically, similar to \cite{chen2022quantizing}), the estimation problem in this paper totally differs from CS. In particular, we will study regression models with multivariate response that can be a \textcolor{black}{vector with considerably large dimension} (LRMR in section \ref{sec3}) or even a huge matrix (see section \ref{sec4}), as in sharp contrast to the scalar measurement $y_k$ in CS. A different point of view is to consider each response scalar of (\ref{intro}). Let $\bm{\Theta}_0 = [\bm{\theta}_{0,1},...,\bm{\theta}_{0,d_2}]$, then the $i$-th entry of $\bm{y}_k$ in (\ref{intro}) can be expressed as $y_{ki}=\bm{x}_k^\top \bm{\theta}_{0,i}+\epsilon_{ki}$. Because $y_{ki}$ only involves the $i$-th column of the desired signal $\bm{\Theta}_0$, it is often referred to as a local measurement and considered to be less informative than the global measurement used in CS (see, e.g., \cite{vaswani2020nonconvex}). As a consequence, the technical ingredients in this work, especially the technique to bound various random terms arising in the proof,   significantly deviate from those in quantized CS.

  From the more statistical side, without considering any data quantization procedure, a lot of statistical procedures have been developed for estimation and prediction in multivariate regression. Among them the most relevant are the regularized ones   that minimize an objective constituted by a loss function and a suitable regularizer, see \cite{negahban2011estimation,yuan2007dimension,kong2019l2rm,rohde2011estimation,koltchinskii2011nuclear,chen2013reduced} for instance. Indeed, the key theoretical achievement of this work is to show the compatibility between the dithered uniform quantizer and the Lasso estimator. That is, Lasso estimator can still achieve near optimal estimation error from data quantized by a uniform quantizer with appropriate random dither.

\subsection{Outline}
The remainder of this paper is organized as follows: we provide the notational conventions and preliminaries in section \ref{sec2}; we propose   our Lasso estimators for quantized LRMR and present the theoretical results in section \ref{sec3}; the main results are then extended to low-rank linear model with matrix response in section \ref{sec4}; we provide experimental results in section \ref{sec5} to validate our theory; we give some remarks to conclude this work in section \ref{sec6}. 

 \section{Preliminaries}\label{sec2}
   \textbf{(Notation){\bf \sffamily.}} We denote matrices and vectors by boldface letters, while scalars by regular letters. We write $[m] = \{1,...,m\}$ for positive integer $m$. 
For vector $\bm{x},\bm{y}\in \mathbb{R}^d$, we work with  the  $\ell_p$ norm $\|\bm{x}\|_p = (\sum_{i\in [d]} |x_i|^p)^{1/p}$, max norm $\|\bm{x}\|_{\infty} = \max_{i\in [d]}|x_i|$, and inner product $\big<\bm{x},\bm{y}\big> = \bm{x}^\top\bm{y}$. For   matrices $\bm{A},\bm{B}$, we work with the transpose $\bm{A}^\top$,  the operator norm $\|\bm{A}\|_{op} $, Frobenius norm $\|\bm{A}\|_F  $, nuclear norm $\|\bm{A}\|_{nu}$ (sum of singular values),  max norm $\|\bm{A}\|_{\infty}= \max_{i,j}|a_{ij}|$, and the inner product $\big<\bm{A},\bm{B}\big>=\mathrm{Tr}(\bm{A}^\top\bm{B})$. The standard Euclidean sphere of $\mathbb{R}^d$ is denoted by $\mathbb{S}^{d-1}= \{\bm{x}\in\mathbb{R}^d:\|\bm{x}\|_2=1\}$. For a random variable $X$, we let $\|X\|_{\psi_2}= \inf\{t>0:\mathbbm{E}\exp(\frac{X^2}{t^2})\leq 2\}$ (resp. $\|X\|_{\psi_1}= \inf\{t>0:\mathbbm{E}\exp(\frac{|X|}{t})\leq 2\}$) be the sub-Gaussian norm (resp. sub-exponential norm), $\|X\|_{L^p} = \big(\mathbbm{E}|X|^p\big)^{1/p}$ be the $L^p$ norm. We represent universal constants by $C$, $c$, $C_i$ or $c_i$, whose value may vary from line to line. We write $T_1 \lesssim T_2$ or $T_1 = O(T_2)$ if $T_1\leq CT_2$; Conversely, $T_1 \gtrsim T_2$ or $T_1 = \Omega(T_2)$ if $T_1 \geq cT_2$. Note that $T_1 \asymp T_2$ if $T_1=O(T_2)$ and $T_2 = \Omega(T_1)$ simultaneously hold. We use $\mathscr{U}(W)$  to denote the uniform distribution over $W$.     We use $\mathrm{vec}(\bm{A})\in \mathbb{R}^{mn\times 1}$ to vectorize a matrix $\bm{A}\in \mathbb{R}^{m\times n}$, while $\mathrm{mat}(.)$ denotes the inverse operator.

\subsection{High-dimensional probability}
A random variable $X$ with  finite $\|X\|_{\psi_2}$ is said to be sub-Gaussian. Note that sub-Gaussian $X$ exhibits exponentially decaying probability tail, i.e., for any $t>0$, \begin{equation}
    \begin{aligned}
    \label{2.1}
        \mathbbm{P}(|X|\geq t)\leq 2\exp\Big(-\frac{ct^2}{\|X\|_{\psi_2}^2}\Big).
    \end{aligned}
\end{equation}
Similarly, $X$ with finite $\|X\|_{\psi_1}$ is sub-exponential and has the following tail bound for any $t>0$  \begin{equation}
\begin{aligned}\label{2.2}
     \mathbbm{P}(|X|\geq t)\leq 2\exp\Big(-\frac{ct}{\|X\|_{\psi_1}}\Big). 
\end{aligned}
\end{equation}
Conversely, both properties in (\ref{2.1}) (resp. (\ref{2.2})) can  characterize the norm $\|.\|_{\psi_2}$ (resp. $\|.\|_{\psi_1}$) up to multiplicative constant, see \cite[Proposition 2.5.2, 2.7.1]{vershynin2018high} for instance. To relate sub-Gaussian norm and sub-exponential norm, one has (see, e.g., \cite[Lem. 2.7.7]{vershynin2018high})\begin{equation}\label{2.3}
    \|XY\|_{\psi_1}\leq \|X\|_{\psi_2}\|Y\|_{\psi_2}.
\end{equation}
For $n$-dimensional random vector $\bm{X}$ we let $\|\bm{X}\|_{\psi_2}=\sup_{\bm{v}\in\mathbb{S}^{n-1}}\|\bm{v}^\top\bm{X}\|_{\psi_2}$.

\subsection{Dithered uniform quantization}
First, we describe the dithered uniform quantization with   $\delta>0$ for an input signal $\bm{x}\in\mathbb{R}^N$ as follows:
\begin{itemize}
    \item Independent of $\bm{x}$, we i.i.d. draw the entries of the random dither $\bm{\tau}\in \mathbb{R}^N$ from some suitable distribution;

    \item Then, we quantize $\bm{x}$   to   $\bm{\dot{x}}= \mathcal{Q}_\delta(\bm{x}+\bm{\tau})$, with $
    \mathcal{Q}_\delta(a) := \delta\big(\big\lfloor\frac{a}{\delta}\big\rfloor+\frac{1}{2}\big) ~(a\in\mathbb{R})$ applied element-wisely.
\end{itemize}

    We adopt the following conventions (as in  \cite{gray1993dithered,gray1998quantization}):
    $\bm{w}:= \bm{\dot{x}}- (\bm{x}+\bm{\tau})$ is the \textbf{quantization error}, and $\bm{\xi}:= \bm{\dot{x}} - \bm{x}$ is the \textbf{quantization noise}.

    The principal properties of  the dithered quantization that underlie our analysis are provided in  Lemma \ref{lem1}.

\begin{lemma} \label{lem1}  {\rm (Theorems 1-2 in \cite{gray1993dithered}){\bf \sffamily.}}
We consider the above dithered uniform quantization: $\bm{x} = [x_i]$ is the input signal,  $\bm{\tau} = [\tau_i]$ is the random dither whose entries are i.i.d. copies of random variable $Y$. We use $\ii$ to denote the complex unit.

\noindent{\rm (a)} {\rm (Quantization Error){\bf \sffamily.}} Let $\bm{w}:=\bm{\dot{x}}-(\bm{x}+\bm{\tau})=[w_i]$ be the quantization error. If $f(u):=\mathbbm{E}(\exp(\ii uY))$ satisfies $f\big(\frac{2\pi  l}{\delta}\big)=0$ for all non-zero integer $l$, then $x_i$ and $w_j$ are independent for all $i,j\in [N]$. Moreover, $\{w_j:j\in [N]\}$ are i.i.d. distributed as $\mathscr{U}\big([-\frac{\delta}{2},\frac{\delta}{2}]\big)$.

\noindent{\rm(b)} {\rm (Quantization Noise){\bf \sffamily.}} Let $\bm{\xi}:= \bm{\dot{x}}-\bm{x}=[\xi_i]$ be the quantization noise. Assume $Z\sim \mathscr{U}[-\frac{\delta}{2},\frac{\delta}{2}]$ is independent of $Y$. Let $g(u):=\mathbbm{E}(\exp(\ii uY))\mathbbm{E}(\exp(\ii u Z))$. Given positive integer $p$, if the $p$-th order derivative $g^{(p)}(u)$ satisfies  $g^{(p)}\big( \frac{2\pi  l}{\delta}\big)=0$ for all non-zero integer $l$, then the $p$-th conditional moment of $\xi_i$ does not depend on $\bm{x}$. More precisely, we have $\mathbbm{E}[\xi_i^p|\bm{x}] = \mathbbm{E}(Y+Z)^p$. 
\end{lemma}
Given quantization level $\delta>0$, in this work we focus on uniform dither $\tau_i \sim \mathscr{U}\big([-\frac{\delta}{2},\frac{\delta}{2}]\big)$ and triangular dither\footnote{This is the term used in prior work, e.g., \cite{gray1993dithered}.} $\tau_i\sim \mathscr{U}\big([-\frac{\delta}{2},\frac{\delta}{2}]\big)+\mathscr{U}\big([-\frac{\delta}{2},\frac{\delta}{2}]\big)$ (i.e., the sum of two independent uniform distribution). From Lemma \ref{lem1}, the following properties are immediate. The proof  can be found in Appendix.

\begin{corollary}
\label{coro1}
In the setting of Lemma \ref{lem1}, if $\tau_i\sim \mathscr{U}\big([-\frac{\delta}{2},\frac{\delta}{2}]\big)$ or $\tau_i\sim \mathscr{U}\big([-\frac{\delta}{2},\frac{\delta}{2}]\big)+\mathscr{U}\big([-\frac{\delta}{2},\frac{\delta}{2}]\big)$, then $x_i$ and $w_j$ are independent $(\forall i,j)$, and $\{w_j:j\}$ are i.i.d. copies of $\mathscr{U}\big([-\frac{\delta}{2},\frac{\delta}{2}]\big)$; In addition, for the triangular dither $\tau_i\sim \mathscr{U}\big([-\frac{\delta}{2},\frac{\delta}{2}]\big)+\mathscr{U}\big([-\frac{\delta}{2},\frac{\delta}{2}]\big)$, the variance of $\xi_i$ is independent of signal; more precisely it holds that $\mathbbm{E}\xi_i^2 = \frac{\delta^2}{4}$. 
\end{corollary}

The benefit of using proper dither (e.g.,   uniform dither) is now clear, i.e., to whiten the quantization noise. For instance, under $\bm{\tau}\sim \mathscr{U}\big([-\frac{\delta}{2},\frac{\delta}{2}]^N\big)$, one has $\mathbbm{E}\bm{\dot{x}}=\mathbbm{E}(\bm{x}+\bm{\tau}+\bm{w})= \mathbbm{E}\bm{x}+\mathbbm{E}\bm{\tau}+ \mathbbm{E}\bm{w}=\mathbbm{E}\bm{x}$. 

 \section{Quantized low-rank multivariate regression}\label{sec3}
 The low-rank multivariate regression (LRMR) model is \begin{equation}\label{3.1}
     \bm{y}_k = \bm{\Theta}_0^\top\bm{x}_k + \bm{\epsilon}_k,~k=1,...,n,
 \end{equation}
 where $\bm{x}_k\in \mathbb{R}^{d_1}$ is the covariate, $\bm{y}_k\in \mathbb{R}^{d_2}$ is the response perturbed by random noise $\bm{\epsilon}_k$, and     $\bm{\Theta}_0\in \mathbb{R}^{d_1\times d_2}$ is the desired parameter. Our goal is to estimate the  $\bm{\Theta}_0$ from $(\bm{x}_k,\bm{y}_k)$'s. We make the following sub-Gaussian assumption. Note that we assume $\mathbbm{E}\bm{x}_k=0$ for simplicity, and the case of "$\mathbbm{E}\bm{x}_k\neq 0$" can be addressed by data centering or including an intercept term in (\ref{3.1}). \textcolor{black}{It should be noted that these distributional assumptions   are standard and commonly adopted for analysing  regularized M-estimators (defined in (\ref{3.3}) shortly) in multiresponse regression problems, see \cite[Coro. 3]{negahban2011estimation}, \cite{hamdi2022low,raskutti2019convex} for instance. Indeed, our Assumption \ref{assumpt1} slightly relaxes the assumptions made in these prior works from Gaussian data to sub-Gaussian data, and note that this relaxation is important for certain cases, e.g., when we work with binary data that cannot be captured by Gaussian distribution.}
 \begin{assumption}\label{assumpt1}
      The covariates $\bm{x}_1,...,\bm{x}_n$ are i.i.d., zero-mean and sub-Gaussian with $\|\bm{x}_k\|_{\psi_2}\leq K$; The covariance matrix $\bm{\Sigma_{xx}}=\mathbbm{E}(\bm{x}_k\bm{x}_k^\top)$ satisfies $\kappa_0\leq\lambda_{\min}(\bm{\Sigma_{xx}})\leq \lambda_{\max}(\bm{\Sigma_{xx}})\leq \kappa_1$ for some $\kappa_1\geq\kappa_0>0$; Independent of $\{\bm{x}_k:k\in [n]\}$, the noise vectors $\bm{\epsilon}_1,...,\bm{\epsilon}_n$ are i.i.d., zero-mean and sub-Gaussian with $\|\bm{\epsilon}_k\|_{\psi_2}\leq E$; $\bm{y}_k$ is generated from (\ref{3.1}) for some low-rank $\bm{\Theta}_0$ satisfying $\rank(\bm{\Theta}_0)\leq r$. 
 \end{assumption}

 Although this multivariate regression model was already intensively studied in the literature (e.g., \cite{negahban2011estimation,giraud2011low,rohde2011estimation}), the novelty of this work lies in the quantization that is inevitable in the era of digital signal processing. In particular, we study "partial quantization" where only the response is quantized,  as well as a more tricky setting of  "complete quantization" where the entire covariate-response pair $(\bm{x}_k,\bm{y}_k)$ is quantized to finite precision. We propose the dithered quantization scheme as follows: 
 
 \begin{itemize}
     \item (Covariate quantization){\bf \sffamily.} Independent of $(\bm{x}_k,\bm{y}_k)$, we i.i.d. draw triangular dither $\bm{\phi}_k\sim\mathscr{U}\big([-\frac{\delta_1}{2},\frac{\delta_1}{2}]^{d_1}\big)+\mathscr{U}\big([-\frac{\delta_1}{2},\frac{\delta_1}{2}]^{d_1}\big)$, and then quantize $\bm{x}_k$ to $\bm{\dot{x}}_k= \mathcal{Q}_{\delta_1}(\bm{x}_k+\bm{\phi}_k)$.
     \item (Response quantization){\bf \sffamily.} Independent of  $(\bm{x}_k,\bm{y}_k)$, we i.i.d. draw uniform dither $\bm{\tau}_k\sim \mathscr{U}\big([-\frac{\delta_2}{2},\frac{\delta_2}{2}]^{d_2}\big)$, and then quantize $\bm{y}_k$ to $\bm{\dot{y}}_k=\mathcal{Q}_{\delta_2}\big(\bm{y}_k+ \bm{\tau}_k\big)$. 
 \end{itemize}
 Note that   $\delta_1=0$ means no quantization on $\bm{x}_k$, thus corresponding to "partial quantization" that only involves response quantization. While  almost all works related works studied response quantization (as   reviewed in section \ref{review}),
 we comment on the necessity of also studying covariate quantization ($\delta_1>0$). For instance, when LRMR appears as a distributed learning problem where the features are transmitted among multiple parties, quantization is often needed for reducing communication cost. Also note that, a mode direct benefit is the lower memory load. 

\subsection{The empirical loss under quantization}

 Using the vector $\ell_1$ norm as regularizer to promote sparsity, Lasso is viewed as a benchmark  procedure for   recovering sparse vector \cite{tibshirani1996regression}. The efficacy of Lasso has extended to the recovery of low-rank matrix by replacing the   $\ell_1$ norm with the   nuclear norm  of a matrix, see \cite{candes2011tight,negahban2011estimation} for instance. Having assumed $\bm{\Theta}_0$ to be low-rank, one can apply similar idea to LRMR and formulate the regularized Lasso recovery program as \begin{equation}\label{3.3}
      \begin{aligned}
          \mathop{\arg\min}\limits_{\bm{\Theta}\in \mathbb{R}^{d_1\times d_2}}~&\underbrace{\frac{1}{n }\sum_{k=1}^n\big\|\bm{y}_k- \bm{\Theta^\top x}_k\big\|_2^2}_{\mathcal{L}(\bm{\Theta})}+\lambda \big\|\bm{\Theta}\big\|_{nu},
      \end{aligned}
 \end{equation}  
where $\mathcal{L}(\bm{\Theta})$ is the $\ell_2$ loss function for data fitting purpose, $\lambda\|\bm{\Theta}\|_{nu}$ is the regularization part for low-rank structure, and $\lambda$ should be tuned to balance data fidelity and the low-rankness. Note that (\ref{3.3})   also falls into the range of M-estimator  \cite{negahban2011estimation,negahban2012restricted}. When a good estimate on $\|\bm{\Theta}_0\|_{nu}$ is available, one can also consider the constrained Lasso \begin{equation} \label{3.4}\mathop{\arg\min}\limits_{\|\bm{\Theta}\|_{nu}\leq R}~ \frac{1}{n}\sum_{k=1}^n\big\|\bm{y}_k- \bm{\Theta}^\top\bm{x}_k\big\|_2^2:=\mathcal{L}(\bm{\Theta}).
\end{equation}
We note that Lasso is known to achieve minimax rate in LRMR, see \cite{negahban2011estimation} for instance.

However, under data quantization one can only access $(\bm{\dot{x}}_k,\bm{\dot{y}}_k)$ for recovery; while the regularizer $\|\bm{\Theta}\|_{nu}$ is unproblematic, one evidently lacks full data for constructing the empirical $\ell_2$ loss $\mathcal{L}(\bm{\Theta})$; modification of $\mathcal{L}({\bm{\Theta}})$ is thus needed. To draw some inspiration, the quite instructive first step is to   calculate the expected $\ell_2$ loss: 
\begin{equation}\nonumber
    \begin{aligned}
        \mathbbm{E}\mathcal{L}(\bm{\Theta}) & =  \mathbbm{E}\|\bm{y}_k- \bm{\Theta}^\top\bm{x}_k\|_2^2\\&\stackrel{(i)}{=} \mathbbm{E}\|\bm{\Theta}^\top \bm{x}_k\|^2 -2\mathbbm{E}\big(\bm{y}_k^\top\bm{\Theta}^\top\bm{x}_k\big)\\
        &=   \big<\bm{\Theta\Theta}^\top,\mathbbm{E}(\bm{x}_k\bm{x}_k^\top)\big>-2\big<\bm{\Theta},\mathbbm{E}(\bm{x}_k\bm{y}_k^\top)\big> \\:&\stackrel{(ii)}{=}  \big<\bm{\Theta\Theta}^\top,\bm{\Sigma_{xx}}\big>-2\big<\bm{\Theta},\bm{\Sigma_{xy}}\big>,
    \end{aligned}
\end{equation}
where $(i)$ holds up to constant that has no effect on the optimization, and in $(ii)$ we introduce the shorthand for the covariance, $\bm{\Sigma_{xx}}=\mathbbm{E}(\bm{x}_k\bm{x}_k^\top)$  and $\bm{\Sigma_{xy}}=\mathbbm{E}(\bm{x}_k\bm{y}_k^\top)$. Therefore, in order to construct a suitable empirical $\ell_2$ loss, we need to find surrogates for $\bm{\Sigma_{xx}}$, $\bm{\Sigma_{xy}}$ based on $(\bm{\dot{x}}_k,\bm{\dot{y}}_k)$.

To facilitate the exposition, we reserve the following notation in subsequent developments: for   quantization of $\bm{x}_k$ with  dither $\bm{\phi}_k$, $\bm{w}_{k1}=\bm{\dot{x}}_k-(\bm{x}_k+\bm{\phi}_k)$ is the quantization error, $\bm{\xi}_{k1}:=\bm{\dot{x}}_k- \bm{x}_k$ is the quantization noise;  for quantization of $\bm{y}_k$ with dither $\bm{\tau}_k$,   $\bm{w}_{k2}:=\bm{\dot{y}}_k- (\bm{y}_k+\bm{\tau}_k)$ stands for the quantization error, while $\bm{\xi}_{k2}:= \bm{\dot{y}}_k-\bm{y}_k$ the quantization noise. We use ${\xi}_{kj,i}$ to denote the $i$-th entry of $\bm{\xi}_{kj}$, and the meanings of notation like ${w}_{kj,i},\phi_{k,i}$ are similar. Now we are ready to present a Lemma that indicates the suitable surrogates of $\bm{\Sigma_{xx}},\bm{\Sigma}_{xy}$. 
\begin{lemma}
    \label{lem2}
    Based on the quantized data $(\bm{\dot{x}}_k,\bm{\dot{y}}_k)$, we let $\bm{\widehat{\Sigma}_{xx}}:= \frac{1}{n}\sum_{k=1}^n\bm{\dot{x}}_k\bm{\dot{x}}_k^\top-\frac{\delta_1^2}{4}\bm{I}_{d_1}$, $\bm{\widehat{\Sigma}_{xy}}= \frac{1}{n}\sum_{k=1}^n\bm{\dot{x}}_k\bm{\dot{y}}_k^\top$, then we have $\mathbbm{E}\bm{\widehat{\Sigma}_{xx}}= \bm{\Sigma_{xx}}$, $\mathbbm{E}\bm{\widehat{\Sigma}_{xy}}=\bm{\Sigma_{xy}}$.
\end{lemma}
\begin{proof}
    We first calculate the easier $\mathbbm{E}\bm{\widehat{\Sigma}_{xy}}=\mathbbm{E}(\bm{\dot{\bm{x}}}_k\bm{\dot{y}}_k^\top)$: \begin{equation}\nonumber
        \begin{aligned}
            &\mathbbm{E}(\bm{\dot{x}}_k\bm{\dot{y}}_k^\top)=\mathbbm{E}(\bm{x}_k+\bm{\phi}_k+\bm{w}_{k1})(\bm{y}_k+\bm{\tau}_k+\bm{w}_{k2})^\top\\&=\mathbbm{E}(\bm{x}_k\bm{y}_k^\top)+\mathbbm{E}(\bm{x}_k\bm{\tau}_k^\top)+\mathbbm{E}(\bm{x}_k\bm{w}_{k2}^\top)\\
            &+\mathbbm{E}(\bm{\phi}_k\bm{y}_k^\top)+\mathbbm{E}(\bm{\phi}_k\bm{\tau}_k^\top)+\mathbbm{E}(\bm{\phi}_k\bm{w}_{k2}^\top)\\&+\mathbbm{E}(\bm{w}_{k1}\bm{y}_k^\top)+\mathbbm{E}(\bm{w}_{k1}\bm{\tau}_k^\top)+\mathbbm{E}(\bm{w}_{k1}\bm{w}_{k2}^\top)\stackrel{(i)}{=}\bm{\Sigma_{xy}},
        \end{aligned}
    \end{equation} 
    where $(i)$ is because in the previous step, all terms but $\mathbbm{E}(\bm{x}_k\bm{y}_k^\top)$ vanish, due to the nice property that $\bm{w}_{k1},\bm{w}_{k2}$ are independent of $(\bm{x}_k,\bm{y}_k)$, $\bm{w}_{k1}\stackrel{iid}{\sim}\mathscr{U}\big([-\frac{\delta_1}{2},\frac{\delta_1}{2}]\big)$, $\bm{w}_{k2}\stackrel{iid}{\sim}\mathscr{U}\big([-\frac{\delta_2}{2},\frac{\delta_2}{2}]\big)$, see Corollary \ref{coro1}. Similarly, to evaluate $\mathbbm{E}\bm{\widehat{\Sigma}_{xx}}=\mathbbm{E}(\bm{\dot{x}}_k\bm{\dot{x}}_k^\top)-\frac{\delta_1^2}{4}\bm{I}_{d_1}$, we calculate $\mathbbm{E}(\bm{\dot{x}}_k\bm{\dot{x}}_k^\top)$ as follows:\begin{equation}
       \begin{aligned}\label{3.6}
           &\mathbbm{E}(\bm{\dot{x}}_k\bm{\dot{x}}_k^\top)=\mathbbm{E}(\bm{x}_k+\bm{\xi}_{k1})(\bm{x}_k+\bm{\xi}_{k1})^\top\\&\stackrel{(i)}{=}\mathbbm{E}(\bm{x}_k\bm{x}_k^\top)+\mathbbm{E}(\bm{\xi}_{k1}\bm{\xi}_{k1}^\top)\stackrel{(ii)}{=}\mathbbm{E}(\bm{x}_k\bm{x}_k^\top)+\frac{\delta_1^2}{4}\bm{I}_{d_1},
       \end{aligned}
    \end{equation}
    where  $(i)$ is due to $\mathbbm{E}(\bm{x}_k\bm{\xi}_{k1}^\top)=\mathbbm{E}(\bm{x}_k\bm{\phi}_{k}^\top)+\mathbbm{E}(\bm{x}_k\bm{w}_{k1}^\top)=0$,   $(ii)$ is because the diagonal entry equals $\mathbbm{E}\xi_{k1,i}^2=\frac{\delta_1^2}{4}$ (Corollary \ref{coro1}), and for $i\neq j$, $\mathbbm{E}(\xi_{k1,i}\xi_{k1,j})=\mathbbm{E}(\phi_{k,i}+w_{k1,i})(\phi_{k,j}+w_{k1,j})=0$, again due to the properties in Corollary \ref{coro1}. The proof is complete.  
\end{proof}
\begin{remark}{\rm (Triangular dither)}
    While the uniform dither  is a quite standard choice in the literature, we comment on the necessity of using triangular dither for $\bm{x}_k$. In essence, this is because in the estimation of $\bm{\Sigma_{xx}}$, the quantized sample covariance contains the bias $\mathbbm{E}(\bm{\xi}_{k1}\bm{\xi}_{k1}^\top)$ (see (\ref{3.6})), which must be removed. However, the diagonal entry $\mathbbm{E}\xi_{k1,i}^2$ remains unknown under the dither of $\mathscr{U}\big([-\frac{\delta_1}{2},\frac{\delta_1}{2}]\big)$, see \cite[Page 3]{gray1993dithered}. Fortunately, by Lemma \ref{lem1}(b), the direct remedy is to use a dither that enjoys quantization noise with signal-independent variance, e.g., $\mathscr{U}\big([-\frac{\delta_1}{2},\frac{\delta_1}{2}]\big)+\mathscr{U}\big([-\frac{\delta_1}{2},\frac{\delta_1}{2}]\big)$. Such triangular dither was also adopted in \cite{chen2022quantizing} when studying covariate quantization in compressed sensing. 
\end{remark}
With all these preparations, we are in a position to specify the empirical loss \begin{equation}
    \begin{aligned}\label{empi}
        \dot{\mathcal{L}}(\bm{\Theta}) =  \big<\bm{\Theta\Theta}^\top,\bm{\widehat{\Sigma}_{xx}}\big>-2\big<\bm{\Theta},\bm{\widehat{\Sigma}_{xy}}\big> .
    \end{aligned}
\end{equation}
Note that $\dot{\mathcal{L}}(\bm{\Theta})$ reduces to the ordinary $\ell_2$ loss $\mathcal{L}(\bm{\Theta})$ (up to additive constant) if $\delta_1=\delta_2=0$. Further combined with the regularizer,   the Lasso recovery procedure for our quantized setting can be proposed. The remainder of this section is devoted to the theoretical analysis of Lasso.

 \subsection{Constrained Lasso}

 First, we study the constrained Lasso where the sparsity is promoted by a ``hard'' constraint. Indeed, we simply substitute the  unknown $\mathcal{L}(\bm{\Theta})$ in (\ref{3.4}) with $\dot{\mathcal{L}}(\bm{\Theta})$, and to focus on estimation problem   per se  we ideally assume the prior estimate is precise, i.e., $R:=\|\bm{\Theta}_0\|_{nu}.$\footnote{One may relax this via more localized arguments as in \cite{plan2016generalized}, which we do not pursue here.} Hence, we formulate the constrained Lasso estimator as \begin{equation}\label{3.8}
     \bm{\widehat{\Theta}}_c = \mathop{\arg\min}\limits_{\|\bm{\Theta}\|_{nu}\leq \|\bm{\Theta^\star}\|_{nu}}~\dot{\mathcal{L}}(\bm{\Theta}),
 \end{equation}
 where $\dot{\mathcal{L}}(\bm{\Theta}) $ is defined as (\ref{empi}). For convenience we define the estimation error $\bm{\widehat{\Delta}}_c:= \bm{\widehat{\Theta}}_c-\bm{\Theta}_0$.

We   begin with two Lemmas that will support the proof of our main Theorem. 
\begin{lemma}
    \label{lem3}
    Assume $\bm{a}_1,...,\bm{a}_n\in\mathbb{R}^{d_1}$ are independent and satisfy $\max_k\|\bm{a}_k\|_{\psi_2}\leq E_1$; $\bm{b}_1,...,\bm{b}_n\in\mathbb{R}^{d_2}$ are independent and satisfy $\max_k\|\bm{b}_k\|_{\psi_2}\leq E_2$. Assume $n\geq d_1+d_2$, then it holds with probability at least $1-2\exp(-c_2(d_1+d_2))$ that,
    \begin{equation}\nonumber
         \big\|\frac{1}{n}\sum_{k=1}^n\big\{\bm{a}_k\bm{b}_k^\top - \mathbbm{E}(\bm{a}_k\bm{b}_k^\top)\big\}\big\|_{op}\leq C_1E_1E_2\sqrt{\frac{d_1+d_2}{n}}.
    \end{equation}  
\end{lemma}
Lemma \ref{lem3} follows similar courses as \cite[Lemma 3]{negahban2011estimation} and involves a standard covering argument. We defer the proof to Appendix.
\begin{lemma}
    \label{lemma4}
    Under Assumption \ref{assumpt1} and our quantization scheme, recall that $\bm{\widehat{\Sigma}_{xx}}=\frac{1}{n}\sum_{k=1}^n\bm{\dot{x}}_k\bm{\dot{x}}_k^\top-\frac{\delta_1^2}{4}\bm{I}_{d_1}$, then the event \begin{equation}\label{3.10}
        \big\|\bm{\widehat{\Sigma}_{xx}}-\bm{\Sigma_{xx}}\big\|_{op}\leq C_1 A_1\Big(\sqrt{\frac{d_1+t}{n}}+\frac{d_1+t}{n}\Big)
    \end{equation} 
    holds with probability at least $1-2\exp(-t)$, where the multiplicative factor is $A_1:=\frac{(K+c_2\delta_1)^2}{{\kappa_0+{\delta_1^2}/4}}$.
\end{lemma}
\begin{proof}
      By (\ref{3.6}) we first note that \begin{equation}
          \begin{aligned}\nonumber
         & \bm{\widehat{\Sigma}_{xx}}-\bm{\Sigma_{xx}}=\frac{1}{n}\sum_{k=1}^n\bm{\dot{x}}_k\bm{\dot{x}}_k^\top -\frac{\delta_1^2}{4}\bm{I}_{d_1}-\mathbbm{E}(\bm{x}_k\bm{x}_k^\top)\\& = \frac{1}{n}\sum_{k=1}^n\bm{\dot{x}}_k\bm{\dot{x}}_k^\top- \mathbbm{E}(\bm{\dot{x}}_k\bm{\dot{x}}_k^\top). 
          \end{aligned}
      \end{equation}
    We then verify the sub-Gaussianity of $\bm{\dot{x}}_k$. Since $\bm{\phi}_k\sim \mathscr{U}\big([-\frac{\delta_1}{2},\frac{\delta_1}{2}]^{d_1}\big)+ \mathscr{U}\big([-\frac{\delta_1}{2},\frac{\delta_1}{2}]^{d_1}\big)$, $\bm{w}_{k1}\sim \mathscr{U}\big([-\frac{\delta_1}{2},\frac{\delta_1}{2}]^{d_1}\big)$, we have  $$\|\bm{\dot{x}}_k\|_{\psi_2}\leq \|\bm{x}_k\|_{\psi_2}+\|\bm{\phi}_k\|_{\psi_2}+\|\bm{w}_{k1}\|_{\psi_2}\leq K+c_2\delta_1$$ for some $c_2$. Moreover, $\lambda_{\min}\big(\mathbbm{E}(\bm{\dot{x}}_k\bm{\dot{x}}_k^\top)\big)= \lambda_{\min}\big(\bm{\Sigma_{xx}}\big)+\frac{\delta_1^2}{4}\geq \kappa_0+\frac{\delta_1^2}{4}$, thus giving $$ \|\bm{v}^\top\bm{\dot{x}}_k\|_{L^2}=\sqrt{\mathbbm{E}(\bm{v}^\top\bm{\dot{x}}_k\bm{\dot{x}}_k^\top\bm{v})}\geq \sqrt{\kappa_0+\frac{\delta_1^2}{4}}$$ for any $\bm{v}\in \mathbb{S}^{d_1-1}$. Therefore, it holds that \begin{equation}
    \begin{aligned}\nonumber
    &\big\| \bm{v}^\top\bm{\dot{x}}_k\big\|_{\psi_2}\leq K+c_2\delta_1 \leq \frac{K+c_2\delta_1}{\sqrt{\kappa_0+{\delta_1^2}/{4}}}\big\|\bm{v}^\top\bm{\dot{x}}_k\big\|_{L^2},
    \end{aligned}
    \end{equation}
    which is just $\sqrt{A_1}\big\|\bm{v}^\top\bm{\dot{x}}_k\big\|_{L^2}$. 
    Finally, we can invoke \cite[Exercise 4.7.3]{vershynin2018high}, which is a well-known estimate in covariance estimation, to arrive at the desired claim. 
\end{proof}
We are now in a position to present our first main Theorem on error bound of (\ref{3.8}). \textcolor{black}{The proof    follows standard lines for analysing regularized M-estimator  (e.g., \cite{negahban2011estimation}), but there are additional random terms to bound due to quantization noise/error, e.g., $(\bm{\xi}_{k1},\bm{\xi}_{k2})$ in $\mathscr{T}_1$ and $\mathscr{T}_2$ in (\ref{3.16}).} 

\begin{theorem}
    \label{theorem1}
    {\rm (Constrained Lasso){\bf \sffamily.}} We consider LRMR under Assumption \ref{assumpt1} and the   quantization procedure described above. We assume the sample complexity $n\gtrsim \big(\frac{A_1}{\kappa_0}\big)^2(d_1+d_2)$, where $A_1$ is the multiplicative factor in (\ref{3.10}). Then for the estimator $\bm{\widehat{\Theta}}_c$ in (\ref{3.8}) we have the following guarantees. 

    \noindent
    {\rm (a) (Partial Quantization){\bf \sffamily.}} If  $\delta_1=0$,  we let $A_2:=\frac{K(E+\delta_2)}{\kappa_0}$, then with probability at least $1-c_2\exp(-c_3(d_1+d_2))$ it holds that \begin{equation}\label{3.11}
        \|\bm{\widehat{\Delta}}_c\|_F\leq C_1A_2\sqrt{\frac{r(d_1+d_2)}{n}}.
    \end{equation}

    \noindent{\rm (b) (Complete Quantization){\bf \sffamily.}} If $\delta_1>0$, $\|\bm{\Theta}_0\|_{op}\leq R$, we let $A_3:=\frac{(K+\delta_1)(E+\delta_2+R\delta_1) }{\kappa_0}$, then with probability at least $1-c_5\exp(-c_6d_1)$ it holds that
    \begin{equation}\label{3.12}
     \|\bm{\widehat{\Delta}}_c\|_F\leq C_4A_3\sqrt{\frac{r(d_1+d_2)}{n}}.
    \end{equation}
\end{theorem}
\begin{proof}
    We begin with the optimality of $\bm{\widehat{\Theta}}_c$
    $$\big<\bm{\widehat{\Theta}}_c\bm{\widehat{\Theta}}_c^\top,\bm{\widehat{\Sigma}_{xx}}\big>-2\big<\bm{\Theta}_c,\bm{\widehat{\Sigma}_{xy}}\big>\leq \big<\bm{\Theta}_0\bm{\Theta}_0^\top,\bm{\widehat{\Sigma}_{xx}}\big>-2\big<\bm{\Theta}_0,\bm{\widehat{\Sigma}_{xy}}\big>.$$
    Then we  use $\bm{\widehat{\Theta}}_c= \bm{\Theta}_0+ \bm{\widehat{\Delta}}_c$ and perform some algebra to arrive at \begin{equation}
        \label{3.13} \big<\bm{\widehat{\Delta}}_c\bm{\widehat{\Delta}}_c^\top,\bm{\widehat{\Sigma}_{xx}}\big> \leq 2\big<\bm{\widehat{\Delta}}_c,\bm{\widehat{\Sigma}_{xy}}-\bm{\widehat{\Sigma}_{xx}}\bm{\Theta}_0\big>,
    \end{equation}
    and the remainder of the proof is essentially to bound both sides of (\ref{3.13}). 

    \noindent{\it Step 1. Bound the left-hand side from below.}

    Due to the scaling $n\gtrsim \big(\frac{A_1}{\kappa_0}\big)^2(d_1+d_2)$, we can invoke Lemma \ref{lemma4} with $t=d_1+d_2$, then $\|\bm{\widehat{\Sigma}_{xx}}-\bm{\Sigma_{xx}}\|_{op}\leq \frac{\kappa_0}{2}$ holds with probability at least $1-2\exp(-(d_1+d_2))$. Combined with $\lambda_{\min}(\bm{\Sigma_{xx}})\geq \kappa_0$, it implies  $\lambda_{\min}(\bm{\widehat{\Sigma}_{xx}})\geq \frac{\kappa_0}{2}$. Therefore, with high probability we have\begin{equation}
            \label{3.14} 
\begin{aligned}&\big<\bm{\widehat{\Delta}}_c\bm{\widehat{\Delta}}_c^\top,\bm{\widehat{\Sigma}_{xx}}\big>=\sum_{j=1}^{d_2}(\bm{\widehat{\Delta}}_c)_{:,j}^\top\bm{\widehat{\Sigma}_{xx}}(\bm{\widehat{\Delta}}_c)_{:,j} \\&\geq \frac{\kappa_0}{2}\sum_{j=1}^{d_2}\|(\bm{\widehat{\Delta}}_c)_{:,j}\|^2_2 = \frac{\kappa_0}{2}\|\bm{\widehat{\Delta}}_c\|_F^2.\end{aligned}
    \end{equation}

    \noindent\textbf{\it Step 2.  Bound the right-hand side from above.}

   Note that \begin{equation}
       \label{addone}\big<\bm{\widehat{\Delta}}_c,\bm{\widehat{\Sigma}_{xy}}-\bm{\widehat{\Sigma}_{xx}\Theta}_0\big>\leq \|\bm{\widehat{\Delta}}_c\|_{nu}\|\bm{\widehat{\Sigma}_{xy}}-\bm{\widehat{\Sigma}_{xx}\Theta}_0\|_{op}
   \end{equation}

    To bound $\|\bm{\widehat{\Delta}}_c\|_{nu}$, we let $\bm{\Theta}_0=\bm{U}_1\bm{\Sigma}\bm{V}_1^\top$ be the (compact) singular value decomposition, where  $\bm{U}_1\in \mathbb{R}^{d_1\times r}$, $\bm{V}_1\in \mathbb{R}^{d_2\times r}$. Also, let $\bm{U}_2\in \mathbb{R}^{d_1\times (d_1-r)} $ (resp. $\bm{V}_2\in \mathbb{R}^{d_2\times(d_2-r)}$) be the orthogonal complement of $\bm{U}_1$ (resp. $\bm{V}_1$). Following   \cite{negahban2012unified}, we define a pair of subspaces as \begin{equation}
        \nonumber
        \begin{aligned}
        &\mathcal{M}=\{\bm{U}_1\bm{A}\bm{V}_1^\top:\forall \bm{A}\in\mathbb{R}^{r\times r}\},\\
        &\overline{\mathcal{M}}= \{\bm{U}_1\bm{A}+\bm{B}\bm{V}_1^\top:\forall \bm{A}\in \mathbb{R}^{r\times d_2},\bm{B}\in \mathbb{R}^{d_1\times r}\}.
        \end{aligned}
    \end{equation}
      For subspace $\mathcal{V}$, we let $\mathcal{V}^\bot$ be its orthogonal complement, and $\mathcal{P}_{\mathcal{V}}(.)$ be the projection   onto $\mathcal{V}$. Then it is not hard to see the decomposibility \cite{negahban2012unified}: $$\|\bm{A}+\bm{B}\|_{nu}=\|\bm{A}\|_{nu}+\|\bm{B}\|_{nu}$$ if $\bm{A}\in \mathcal{M}$, $\bm{B}\in \overline{\mathcal{M}}^\bot$. Now   we can deduce that \begin{equation}
        \begin{aligned}\label{3.15}
            & \|\bm{\Theta}_0+\bm{\widehat{\Delta}}_c\|_{nu} =\|\bm{\Theta}_0+\mathcal{P}_{\overline{\mathcal{M}}}\bm{\widehat{\Delta}}_c+\mathcal{P}_{\overline{\mathcal{M}}^\bot}\bm{\widehat{\Delta}}_c\|_{nu}\\
            & {\geq} \|\bm{\Theta}_0+\mathcal{P}_{\overline{\mathcal{M}}^\bot}\bm{\widehat{\Delta}}_c\|_{nu}-\|\mathcal{P}_{\overline{\mathcal{M}}}\bm{\widehat{\Delta}}_c\|_{nu} \\&{\geq}   \|\bm{\Theta}_0\|_{nu}+\|\mathcal{P}_{\overline{\mathcal{M}}^\bot}\bm{\widehat{\Delta}}_c\|_{nu}-\|\mathcal{P}_{\overline{\mathcal{M}}}\bm{\widehat{\Delta}}_c\|_{nu}.
        \end{aligned}
    \end{equation}
    Combined with the constraint $\|\bm{\widehat{\Theta}}_0 + \bm{\widehat{\Delta}}_c\|_{nu}=\|\bm{\widehat{\Theta}}_c\|_{nu}\leq \|\bm{\Theta}_0\|_{nu}$, we obtain  $\|\mathcal{P}_{\overline{\mathcal{M}}^\bot}\bm{\widehat{\Delta}}_c\|_{nu}\leq \|\mathcal{P}_{\overline{\mathcal{M}}}\bm{\widehat{\Delta}}_c\|_{nu}$. Thus, \begin{equation}
        \begin{aligned}\nonumber
        \|\bm{\widehat{\Delta}}_c\|_{nu}&\leq \|\mathcal{P}_{\overline{\mathcal{M}}^\bot}\bm{\widehat{\Delta}}_c\|_{nu}+\|\mathcal{P}_{\overline{\mathcal{M}}}\bm{\widehat{\Delta}}_c\|_{nu}\\&\leq 2\|\mathcal{P}_{\overline{\mathcal{M}}}\bm{\widehat{\Delta}}_c\|_{nu}\leq 2\sqrt{2r}\|\bm{\widehat{\Delta}}_c\|_F.
        \end{aligned}
    \end{equation} The last inequality is because $\rank(\bm{A})\leq 2r$ if $\bm{A}\in \overline{\mathcal{M}}$, and it always holds that $\|\bm{A}\|_{nu}\leq \rank(\bm{A})\|\bm{A}\|_F$.

It remains to bound $\|\bm{\widehat{\Sigma}_{xy}}-\bm{\widehat{\Sigma}_{xx}\Theta}_0\|_{op}$. We first plug in $\bm{\widehat{\Sigma}_{xx}}$, $\bm{\widehat{\Sigma}_{xy}}$, and further $$\bm{\dot{x}}_k= \bm{x}_k+\bm{\xi}_{k1},~\bm{\dot{y}}_k=\bm{y}_k+ \bm{\xi}_{k2}= \bm{\Theta}_0^\top\bm{x}_k+\bm{\epsilon}_k,$$ some algebra yields 
\begin{equation}
  \begin{aligned}\label{3.16}
      & \bm{\widehat{\Sigma}_{xy}}-\bm{\widehat{\Sigma}_{xx}\Theta}_0 =  {\frac{1}{n}\sum_{k=1}^n (\bm{x}_k+\bm{\xi}_{k1})(\bm{\epsilon}_k+\bm{\xi}_{k2})^\top} -\\&  {\frac{1}{n}\sum_{k=1}^n\Big(\bm{x}_k\bm{\xi}_{k1}^\top+(\bm{\xi}_{k1}\bm{\xi}_{k1}^\top-\frac{\delta_1^2}{4}\bm{I}_{d_1})\Big)\bm{\Theta}_0}:=\mathscr{T}_1-{\mathscr{T}_2}.
  \end{aligned}
\end{equation}
Thus, $\|\bm{\widehat{\Sigma}_{xy}}-\bm{\widehat{\Sigma}_{xx}\Theta}_0\|_{op}\leq \|\mathscr{T}_1\|_{op}+\|\mathscr{T}_2\|_{op}$.

\noindent{(a)} We consider   the case of partial quantization  ($\delta_1=0$). In this case $\bm{\xi}_{k1}=0$, so $\mathscr{T}_2=0$ and we only need to bound $\|\mathscr{T}_1\|_{op}$ with $\mathscr{T}_1=\frac{1}{n}\sum_k \bm{x}_k(\bm{\epsilon}_k+\bm{\xi}_{k2})^\top$. Note that $$\|\bm{\epsilon}_k+\bm{\xi}_{k2}\|_{\psi_2}\leq \|\bm{\epsilon}_k\|_{\psi_2}+\|\bm{\xi}_{k2}\|_{\psi_2}\lesssim E+\delta_2,$$ $\|\bm{x}_k\|_{\psi_2}\leq K$, $\mathbbm{E}(\bm{x}_k\bm{\epsilon}_k^\top)=\mathbbm{E}(\bm{x}_k\bm{\xi}_{k2}^\top)=0$, Lemma \ref{lem3} guarantees the following to hold with probability at least $1-2\exp(-c_1(d_1+d_2))$\begin{equation}\label{add1}
    \|\mathscr{T}_1\|_{op}\lesssim K(E+\delta_2)\sqrt{\frac{d_1+d_2}{n}}.
\end{equation}  
Overall, we have \begin{equation}
    \begin{aligned}\label{3.17}\big<\bm{\widehat{\Delta}}_c,\bm{\widehat{\Sigma}_{xy}}&-\bm{\widehat{\Sigma}_{xx}}\bm{\Theta}_0\big>\\&\lesssim K(E+\delta_2)  \|\bm{\widehat{\Delta}}_c\|_F \sqrt{\frac{r(d_1+d_2)}{n}}.\end{aligned}
\end{equation} 
The result of part (a) follows by putting (\ref{3.14}) and (\ref{3.17}) into  (\ref{3.13}).

\noindent{(b)} We then consider the complete quantization case   ($\delta_1>0$). Similarly to (a), we have the bound $$\|\mathscr{T}_1\|_{op}\lesssim (K+\delta_1)(E+\delta_2)\sqrt{\frac{d_1+d_2}{n}}.$$ So it remains to bound $\|\mathscr{T}_2\|_{op}$. Since $\|\bm{\Theta}_0\|_{op}\leq R$, and by Lemma \ref{lem3}, with the promised probability we have\begin{equation}\label{3.19}
   \begin{aligned} &\|\mathscr{T}_2\|_{op}\leq R \Big\| \frac{1}{n}\sum_{k=1}^n \big(\bm{\xi}_{k1}\bm{\xi}_{k1}^\top - \frac{\delta_1^2}{4}\bm{I}_{d_1}\big)\Big\|_{op}\\&+R  \Big\|\frac{1}{n}\sum_{k=1}^n \bm{x}_k\bm{\xi}_{k1}^\top\Big\|_{op}\lesssim R\delta_1(K+\delta_1)\sqrt{\frac{d_1+d_2}{n}}.
\end{aligned}\end{equation}
By putting pieces similarly, we conclude the proof.
\end{proof}
Several remarks are in order.
\begin{remark}
    {\rm(Prediction error)} As presented in Theorem \ref{theorem1}, we will focus on the estimation of $\bm{\Theta}_0$ in this work, whereas in regression one may also be interested in the prediction performance. From $\|\bm{\widehat{\Delta}}_c\|_F$, the bound for prediction error is indeed immediate. For instance, when $\delta_1=0$, because with high probability $\|\bm{\widehat{\Sigma}_{xx}}-\bm{\Sigma_{xx}}\|_{op}\leq \frac{\kappa_0}{2}$ and $\lambda_{\max}(\bm{\widehat{\Sigma}_{xx}})\leq \kappa_1$, one has $$\frac{1}{n}\sum_{k=1}^n \|\bm{\widehat{\Delta}}_c^\top\bm{x}_k\|^2_2=\big<\bm{\widehat{\Sigma}_{xx}},\bm{\widehat{\Delta}}_c\bm{\widehat{\Delta}}_c^\top\big>=O\big(\kappa_1\|\bm{\widehat{\Delta}}_c\|^2\big).$$
\end{remark}

\begin{remark}\label{ls}
{\rm(Compared to the least squares estimation)} The ordinary least squares (OLS) estimator $\bm{\widehat{\Theta}}_{LS}$ is to minimize $\dot{\mathcal{L}}(\bm{\Theta})$ over   $\bm{\Theta}\in \mathbb{R}^{d_1\times d_2}$ without the nuclear norm constraint. This amounts to estimating $d_2$ columns of $\bm{\Theta}_0$ separately   without utilizing their correlations. Under similar assumptions on covariate and noise, one can easily show $\|\bm{\widehat{\Theta}}_{LS}-\bm{\Theta}_0\|_F$ scales as $O\big(\sqrt{\frac{d_1d_2}{n}}\big)$, which is essentially inferior to $O\big(\sqrt{\frac{r(d_1+d_2)}{n}}\big)$ in the case of $r\ll \min\{d_1,d_2\}$. This illustrates the benefit of incorporating the low-rank priori on $\bm{\Theta}_0$, which will be complemented by numerical example later (Figure \ref{fig5}). 
\end{remark} 
 
\begin{remark}
    \label{rem2}
    {\rm(Minimax optimality and the role of quantization)}
   The non-asymptotic error bound $O\big(\sqrt{\frac{r(d_1+d_2)}{n}}\big)$  is minimax optimal compared to the information-theoretic lower bound in \cite[Theorem 5]{rohde2011estimation} (also see \cite[Remark 11]{fan2021shrinkage}, \cite[Fact 1]{giraud2011low}, \cite[Page 12]{bunea2011optimal} for alternative statements). In fact, the quantization \textcolor{black}{does not affect the order of $(n,r,d_1,d_2)$ in the sample complexity and error bounds} but only slightly worsens the multiplicative factors, i.e., $A_1$ in  $n\gtrsim \big(\frac{A_1}{\kappa_0}\big)^2(d_1+d_2)$, $A_2$ in (\ref{3.11}) and $A_3$ in (\ref{3.12}). \textcolor{black}{Thus, in a regime where the quantization levels $\delta_1,\delta_2$ are fixed, our result   matches   the one in a   case  without quantization up to multiplicative constant.} 
   In addition,   $\delta_1$ and $E$ are on equal footing in $A_2$, hence the role of partial quantization can be nicely  interpreted as additional sub-Gaussian noise. This extends similar findings in \cite{thrampoulidis2020generalized,sun2022quantized,xu2020quantized} from compressed sensing to LRMR. \textcolor{black}{Further, a useful perspective is that the result for the setting without quantization can be recovered by letting $\delta_1,\delta_2=0$. For instance, when $\delta_2=0$ the bound in Theorem \ref{theorem1} reads as $O\big(\frac{KE}{\kappa_0}\sqrt{\frac{r(d_1+d_2)}{n}}\big)$, thus agreeing with the bound in \cite[Coro. 3]{negahban2011estimation}. The above discussions regarding the role of quantization remain valid for our subsequent results.} 
\end{remark}

 \subsection{Regularized Lasso}
 Since prior estimate on $\|\bm{\Theta}_0\|_{nu}$ is often unavailable, a more practically appealing recovery procedure is the following regularized Lasso given by \begin{equation}\label{lrmrregu}
     \bm{\widehat{\Theta}}_p = \mathop{\arg\min}\limits_{\bm{\Theta}\in\mathbb{R}^{d_1\times d_2}} ~  \dot{\mathcal{L}}(\bm{\Theta})+\lambda\|\bm{\Theta}\|_{nu},
 \end{equation} 
and we let $\bm{\widehat{\Delta}}_p$ be the estimation error. By properly tuning $\lambda$, the Regularized Lasso estimator $\bm{\widehat{\Theta}}_p$ achieves the same error rate as the previous $\bm{\widehat{\Theta}}_c$. 

\begin{theorem}
    \label{thm2}{\rm   (Regularized Lasso){\bf \sffamily.}} 
    We consider LRMR under Assumption \ref{assumpt1} and the    quantization procedure described above. We assume the scaling $n\gtrsim \big(\frac{A_1}{\kappa_0}\big)^2(d_1+d_2)$, where $A_1$ is the multiplicative factor in (\ref{3.10}). Then for the estimator $\bm{\widehat{\Theta}}_p$ in (\ref{lrmrregu}) we have the following guarantees.  

    \noindent
    {\rm (a) (Partial Quantization){\bf \sffamily.}} If $\delta_1=0$, we let $A_4:={K(E+\delta_2)} $. Set $\lambda=C_1A_4\sqrt{\frac{d_1+d_2}{n}}$ with sufficiently large $C_1$, then with probability at least $1-c_3\exp(c_4(d_1+d_2))$ it holds that 
\begin{equation}
    \label{3.20}
   \|\bm{\widehat{\Delta}}_p\|_F\leq C_2\kappa_0^{-1}A_4\sqrt{\frac{r(d_1+d_2)}{n}} .
\end{equation}

     \noindent{\rm (b) (Complete Quantization){\bf \sffamily.}}  If $\delta_2>0$, $\|\bm{\Theta}_0\|_{op}\leq R$, we let $A_5:= {(K+\delta_1)(E+\delta_2+R\delta_1) } $. Set $\lambda = C_5A_5\sqrt{\frac{ d_1+d_2 }{n}}$ with sufficiently large $C_5$, then with probability at least $1-c_7\exp(-c_8d_1)$ it holds that \begin{equation}\nonumber
         \|\bm{\widehat{\Delta}}_p\|_F\leq C_6\kappa_0^{-1}A_5\sqrt{\frac{r(d_1+d_2)}{n}}.
     \end{equation} 
\end{theorem}
 By using  some standard  analyses for regularizer M-estimator (e.g., see \cite{negahban2011estimation}), the proof of Theorem \ref{thm2} follows similar lines of Theorem \ref{theorem1}. We defer the   proof to Appendix.

It is clear that we need an additional constraint on $\|\bm{\Theta}_0\|_{op}$ for the cases of complete quantization in Theorems \ref{theorem1}-\ref{thm2}, while this is not needed when we have access to the full-precision covariate. The following remark elaborates on this point. 

\begin{remark}\label{rem3}
{\rm (The norm constraint of $\bm{\Theta}_0$)} When there is error in covariate, a norm constraint on the true parameter $\bm{\Theta}_0$ seems indispensable rather than an artifact from the   proof technique. The main reason  is that the error in covariate   propagates  along  the true parameter, and hence its overall contribution to the response is proportional to $\bm{\Theta}_0$. Note that similar observation was also made in   \cite[Section 3.2]{loh2013regularized} for corrected linear regression where the covariates suffer from zero-mean random noise with known covariance matrix.
\end{remark}

\section{quantized low-rank linear regression model with Matrix Response}\label{sec4}

The proposed quantization scheme  enjoys broader applicability, as we will show in this section that the dithered quantizer can be similarly applied to the problem of low-rank linear regression   (L2RM) with matrix response \cite{kong2019l2rm}. In particular, such regression model finds application in imaging genetics, with matrix responses representing the weighted or binary adjacency matrix of a finite graph that characterizes structural or functional connectivity pattern, while the covariates are a set of genetic markers \cite{thompson2013genetics,medland2014whole,wen2020co}. We would also like to note some recent advances on variable selection \cite{hao2021optimal} and covariance estimation \cite{zhang2022covariance} for matrix-valued data.

Following the notation in \cite{kong2019l2rm}, L2RM with matrix response can be formulated as \begin{equation}\label{matrep}\bm{Y}_k=\sum_{i=1}^s x_{ki}\bm{\Theta}^{(i)}_0 +\bm{E}_k,~k=1,...,n,\end{equation}
where $\bm{x}_k=[x_{k1},...,x_{ks}]^\top$ is the covariate, $\bm{\Theta}_0^{(i)}\in \mathbb{R}^{p\times q}$ are the true coefficient matrices, $\bm{E}_k,\bm{Y}_k\in \mathbb{R}^{p\times q}$ are respectively the noise matrix and response. Our goal is to estimate $\bm{\Theta}_0=[\bm{\Theta}_0^{(1)},...,\bm{\Theta}_0^{(s)}]\in \mathbb{R}^{p\times (sq)}$ under moderately large $s$      but $p,q$ that can be extremely huge.\footnote{In fact,  $s$ in real applications can also be very large. For dimension reduction, \cite{kong2019l2rm} assumed   $\bm{\Theta}_0^{(i)}=0$ for most $i$'s and developed a screening method to estimate those $i$'s with non-zero $\bm{\Theta}_0^{(i)}$. We focus on the estimation after this screening step.} \textcolor{black}{Analogously to Assumption \ref{assumpt1}, for analysing  the nuclear norm regularized M-estimator (see (\ref{regulasso}) below), we make the following sub-Gaussian data assumptions that relax  the Gaussian ones in \cite[(A9)-(A11)]{kong2019l2rm}.}
\begin{assumption}
    \label{assumption2}
    The assumptions on covariates $\bm{x}_k$'s are the same as Assumption \ref{assumpt1}; Independent of $\bm{x}_k$'s, the noise matrices $\bm{E}_1,...,\bm{E}_n$ are i.i.d., zero-mean and sub-Gaussian with $\|\bm{E}_k\|_{\psi_2}:=\sup_{\bm{u}\in \mathbb{S}^{p-1}}\sup_{\bm{v}\in\mathbb{S}^{q-1}}\|\bm{u}^\top \bm{E}_k\bm{v}\|_{\psi_2}\leq E$; the matrix responses $\bm{Y}_k$'s are generated from (\ref{matrep}) for some $\bm{\Theta}_0$ satisfying $\sum_{i=1}^s\rank(\bm{\Theta}_0^{(i)})\leq r$. 
\end{assumption}
Similarly, the dithered quantization for $(\bm{x}_k,\bm{Y}_k)$ is as follows: $\bm{\dot{x}}_k=\mathcal{Q}_{\delta_1}(\bm{x}_k+\bm{\phi}_k)$ for triangular random dither $\bm{\phi}_k$; $\bm{\dot{Y}}_k=\mathcal{Q}_{\delta_2}(\bm{Y}_k+\bm{\tau}_k)$ with uniform random dither $\bm{\tau}_k\sim \mathscr{U}\big([-\frac{\delta_2}{2},\frac{\delta_2}{2}]^{p\times q}\big)$. To be concise we only consider the more practical regularized Lasso. Based on the full data $(\bm{x}_k,\bm{Y}_k)$, \cite{kong2019l2rm} proposed the unconstrained convex program that minimizes $$ \underbrace{\frac{1}{n }\sum_{k=1}^n\Big\|\bm{Y}_k-\sum_{i=1}^sx_{ki}\bm{\Theta}^{(i)}\Big\|_F^2}_{{\mathcal{L}}_1(\bm{\Theta})} + \lambda \cdot \sum_{i=1}^s\big\|\bm{\Theta}^{(i)}\big\|_{nu}$$  
over $\bm{\Theta}=[\bm{\Theta}^{(1)},...,\bm{\Theta}^{(s)}]$, where $\sum_i\|\bm{\Theta}^{(i)}\|_{nu}$ is the regularizer that incorporates the low-rankness structures of $\bm{\Theta}^{(i)}$'s. However, in our quantized regime one only observes $(\bm{\dot{x}}_k,\bm{\dot{Y}}_k)$ ($\bm{\dot{x}}_k=\bm{x}_k$ in partial quantization with $\delta_1=0$), modification of $\mathcal{L}_1(\bm{\Theta})$ is needed. By vectorization we first reformulate (\ref{matrep}) as $\mathrm{vec}(\bm{Y}_k)= \sum_{i=1}^sx_{ki}\cdot \mathrm{vec}(\bm{\Theta}_0^{(i)})+\mathrm{vec}(\bm{E}_k)$. Here, for $\bm{\Theta}=[\bm{\Theta}^{(1)},...,\bm{\Theta}^{(s)}]$ we define the rearrangement $\bm{\widetilde{\Theta}}$ as \begin{equation}
    \bm{\widetilde{\Theta}}= \begin{bmatrix}\mathrm{vec}(\bm{\Theta} ^{(1)})^\top \\\vdots \\ \mathrm{vec}(\bm{\Theta} ^{(s)})^\top\end{bmatrix}\in \mathbb{R}^{s\times pq},\label{rearr}
\end{equation}
then we have \begin{equation}
\label{vectorization}
    \mathrm{vec}(\bm{Y}_k)=\bm{\widetilde{\Theta}}_0^\top \bm{x}_k+ \mathrm{vec}(\bm{E}_k)
\end{equation} that agrees with (\ref{3.1}). Now we can employ the prior developments --- similar to (\ref{empi}) we let $\bm{\widehat{\Sigma}_{xx}}:=\frac{1}{n}\sum_{k=1}^n\bm{\dot{x}}_k\bm{\dot{x}}_k^\top-\frac{\delta_1^2}{4}\bm{I}_s$, $\bm{\widehat{\Sigma}_{xy}}= \frac{1}{n}\sum_{k=1}^n\bm{\dot{x}}_k\mathrm{vec}(\bm{\dot{Y}}_k)^\top$, and  then 
change $\mathcal{L}_1(\bm{\Theta})$ to  \begin{equation}
\begin{aligned}\nonumber
\dot{\mathcal{L}}_1(\bm{\Theta}) =  \big<\bm{\widetilde{\Theta}}\bm{\widetilde{\Theta}}^\top,\bm{\widehat{\Sigma}}_{\bm{xx}}\big>-2\big<\bm{\widetilde{\Theta}},\bm{\widehat{\Sigma}_{xy}}\big>,
\end{aligned}
\end{equation}
which can be constructed from the quantized data. Combining these pieces, we are in a position to define the Lasso estimator:
\begin{equation}\label{regulasso}\begin{aligned}
     &\bm{\widehat{\Theta}} = \mathop{\arg\min}\limits_{\bm{\Theta}} ~  \dot{\mathcal{L}}_1(\bm{\Theta})+\lambda\cdot \sum_{i=1}^s\|\bm{\Theta}^{(i)}\|_{nu}\\&\text{subject to }\bm{\Theta}=[\bm{\Theta}^{(1)},...,\bm{\Theta}^{(s)}]\in \mathbb{R}^{p\times sq}\end{aligned}
 \end{equation} 
 We have the following theoretical guarantee for $\bm{\widehat{\Theta}}$.
 \begin{theorem}
     \label{theorem3}
     {\rm   (Regularized Lasso){\bf \sffamily.}} We consider L2RM with matrix response under Assumption \ref{assumption2} and the quantization procedure described above. We assume
     the scaling $n\gtrsim  \max\{s,p,q\}$ for some sufficiently large hidden constant and $ \log s=O(p+q)$. Then for estimator $\bm{\widehat{\Theta}}$ in (\ref{regulasso}) we have the following guarantees.

     \noindent
    {\rm (a) (Partial Quantization){\bf \sffamily.}} If $\delta_1=0$, we let $A_6:= K(E+\delta_2)$. Set $\lambda=C_1A_6\sqrt{\frac{p+q}{n}}$ with sufficiently large $C_1$, then with probability at least $1-\exp(-s)-c_3\exp(-c_4(p+q))$ it holds that \begin{equation}
        \label{matres_part}
        \|\bm{\widehat{\Theta}}-\bm{\Theta}_0\|_F \leq C_2\kappa_0^{-1}A_6\sqrt{\frac{r(p+q)}{n}}.
    \end{equation}

    \noindent{\rm (b) (Complete Quantization){\bf \sffamily.}} If $\delta_1>0$, we further assume ${\sum_{i=1}^s\|\bm{\Theta}_0^{(i)}\|_{op}^2}\leq R^2$ for some $R>0$, $s=O(p+q)$, and then let $A_7 := (K+\delta_1)(E+\delta_2+R\delta_1)$. Set $\lambda = C_2A_7\sqrt{\frac{p+q}{n}}$ with sufficiently large $C_2$, then with probability at least $1-\exp(-s)-c_1\exp(-c_2(p+q))$ it holds that 
    \begin{equation}
        \label{matres_com}
        \|\bm{\widehat{\Theta}}-\bm{\Theta}_0\|_F \leq C_3\kappa_0^{-1}A_7\sqrt{\frac{r(p+q)}{n}}.
    \end{equation}
 \end{theorem}
\textcolor{black}{Setting $\delta_2=0$ in Theorem \ref{theorem3}(a) exactly recovers \cite[Theorem 5]{kong2019l2rm}. While beyond the range of \cite{kong2019l2rm}, our results clearly display how the dithered quantization affects the error bounds, i.e., slightly worse multiplicative factors ($A_6,A_7$). Specifically, when $\delta_1$ and $\delta_2$ are chosen and then fixed, the estimation error still scales as $O\big(\sqrt{\frac{r(p+q)}{n}}\big)$, which matches the case without quantization up to multiplicative constant.}

\textcolor{black}{There are   some technical differences between our proof and the one for \cite[Theorem 5]{kong2019l2rm}. First, because we assume sub-Gaussian $(\bm{x}_k,\bm{E}_k)$ rather than the Gaussian ones as in \cite{kong2019l2rm}, different arguments are required to proceed the proof. More specifically, Gaussian $(\bm{x}_k,\bm{E}_k)$ enables  \cite[Theorem 5]{kong2019l2rm}  to use techniques from \cite{raskutti2019convex} like Anderson’s comparison inequality (see \cite[Lemma 4]{kong2019l2rm}) and  tail bound of $\chi^2$ random variable to bound $\|\frac{1}{n}\sum_{k=1}^n x_{ki}\bm{E}_k\|_{op}$. In contrast,  this term is bounded via Lemma \ref{lemma5} in   (\ref{reviseadd3});  besides handling sub-Gaussian $(\bm{x}_k,\bm{E}_k)$, Lemma \ref{lemma5} itself represents a cleaner way to bound this random term compared to the arguments in \cite{kong2019l2rm}. Second, in the "complete quantization" case, due to error in the covariate, there appears an additional random term in (\ref{reviseadd1}), (\ref{reviseadd2}), and  to bound it we need to further assume $\sum_{i=1}^s\|\bm{\Theta}_0^{(i)}\|_{op}^2\leq R^2$ (as explained in Remark \ref{rem3}).
We defer the detailed proof to Appendix.}

We give the following remark to compare (\ref{regulasso})  with the ordinary least squares method and the Lasso for LRMR based on the reformulation (\ref{vectorization}). 

\begin{remark}
{\rm (Compared to OLS and LRMR via vectorization)}
 For the estimator $\bm{\widehat{\Theta}}_{LS}$ defined by minimizing the empirical $\ell_2$ loss over $\bm{\Theta}\in \mathbb{R}^{p\times sq}$, the error $\|\bm{\widehat{\Theta}}_{LS}-\bm{\Theta}_0\|_F$ would scale as $O\big(\sqrt{\frac{spq}{n}}\big)$ even without quantization. By contrast, the deduced $O\big(\sqrt{\frac{r(p+q)}{n}}\big)$ can be essentially better when $r\ll s\min\{p,q\}$. This illustrates the benefit of incorporating the low-rank structure. Moreover,     if we impose low-rankness on $\bm{\widetilde{\Theta}}_0$ after vectorization (\ref{vectorization}), then by Theorem \ref{thm2} the estimation error   scales as $O\big(\sqrt{\frac{r_1pq}{n}}\big)$ (here, $r_1=\rank(\bm{\widetilde{\Theta}}_0)$), which still suffers from the extremely large $pq$. Thus, the method in this section (also, as in \cite{kong2019l2rm}) achieves more effective dimension reduction in the case of matrix response.
\end{remark}


 \section{Experimental Results}\label{sec5}
In this section we provide experimental results to support and demonstrate our theoretical results. Otherwise specified, each data point is set to be the mean value of 50 independent trials. 
\subsection{Simulations with synthetic data}
We first present simulation results on synthetic data. Our main purpose is to verify the established error rates, specifically $O\big(A\sqrt{\frac{r(d_1+d_2)}{n}}\big)$ in Theorems \ref{theorem1}-\ref{thm2} and $O\big(A\sqrt{\frac{r(p+q)}{n}}\big)$ in Theorem \ref{theorem3}, are in the correct order for characterizing the Lasso estimation errors. In particular, the dithered quantization only results in slightly larger multiplicative factor $A$. We will also demonstrate the important role played by the random dithering.
\subsubsection{Constrained Lasso   for quantized LRMR}
  To simulate the setting of quantized LRMR we generate the low-rank underlying $\bm{\Theta}_0\in\mathbb{R}^{d_1\times d_2}$ as follows: we first generate $\bm{\Theta}_1\in\mathbb{R}^{d_1\times r}$, $\bm{\Theta}_2\in\mathbb{R}^{r\times d_2}$ with i.i.d. standard Gaussian entries, and then use a rescaled version of $ \bm{\Theta}_1\bm{\Theta}_2$ (with unit Frobenius norm) as $\bm{\Theta}_0$. To simulate the sub-Gaussian data in Assumption \ref{assumpt1}, for simplicity, we use $\bm{x}_k\sim\mathcal{N}(0,\bm{I}_{d_1})$ and   $\bm{\epsilon}_k\sim\mathcal{N}(0,0.1\cdot\bm{I}_{d_2})$. The constrained Lasso is fed with $R=\|\bm{\Theta}_0\|_{nu}$ and optimized by an algorithm based on alternating direction method of multipliers (ADMM)   \cite{boyd2011distributed}. To verify and demonstrate the error rate of $O\big(A\sqrt{\frac{r(d_1+d_2)}{n}}\big)$, we test different choices of $(d_1,d_2,r,\delta_1,\delta_2)$ under $n=1000:500:3500$, with the log-log error plots displayed in Figure \ref{fig1}. Firstly, the experimental curves are   aligned with the dashed line that represents the decreasing rate of $n^{-1/2}$, thus confirming the order regarding the sample size. Then, to illustrate that quantization merely affects multiplicative factors, we   compare the curves of $\delta_2=0.2,0.3,0.4$ in Figure \ref{fig1}(a)  (partial quantization) and the curves for  $\delta_1=\delta_2=0.2,0.3,0.4$ in Figure \ref{fig1}(b) (complete quantization). Note that   these curves     are still parallel to each other, while the ones with larger $\delta_i$ are higher, which is consistent with our theory.
Moreover, we note that increasing $d_1$ (from $50$ to $70$) or $r$ (from $5$ to $8$) also leads to larger estimation errors. This is also predicted by  the theoretical bound $O\big(\sqrt{\frac{r(d_1+d_2)}{n}}\big)$, that is, LRMR with more coefficients or weaker low-rank structure is harder.

\begin{figure}[ht!]
    \centering
    \includegraphics[scale = 0.43]{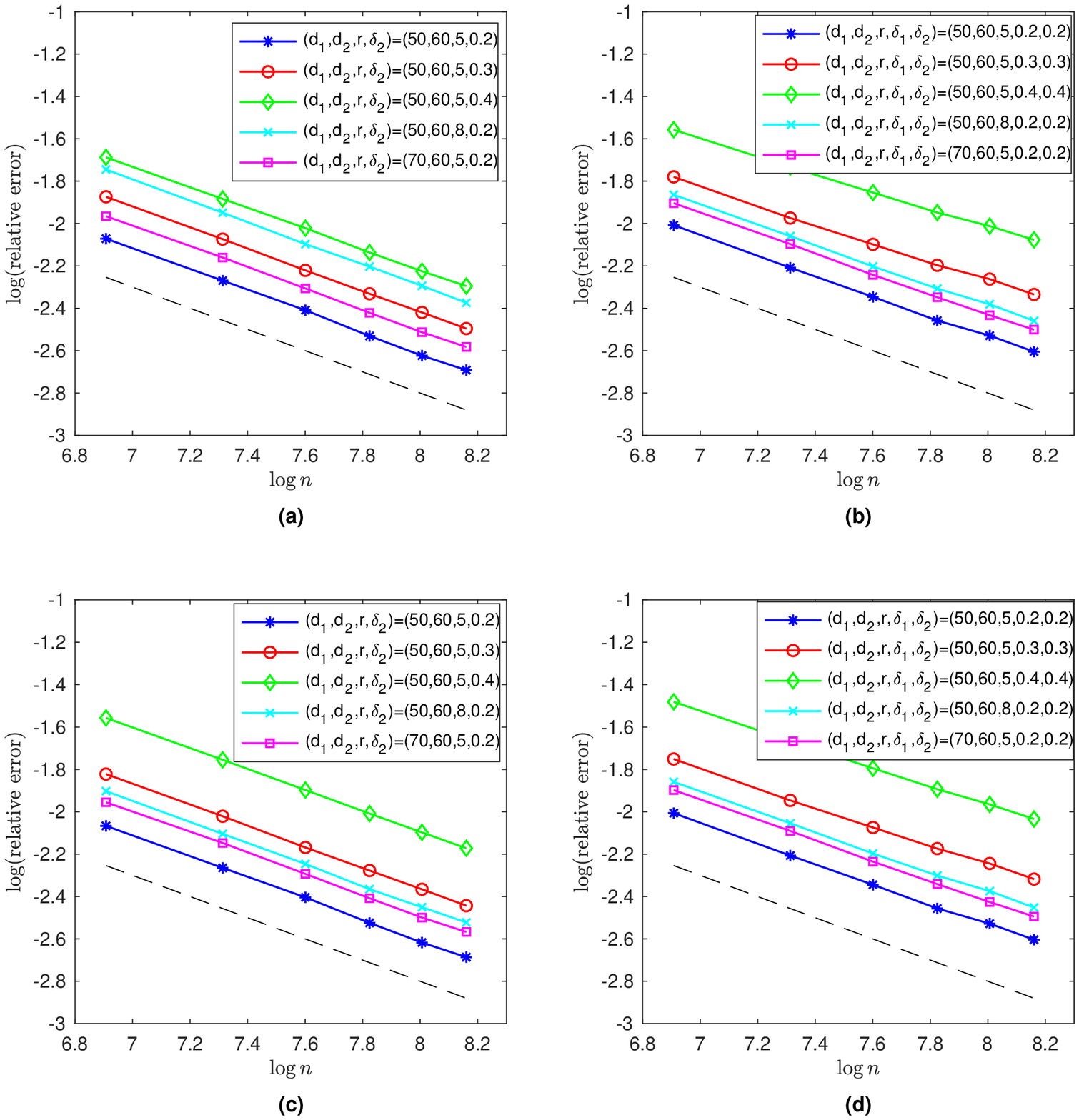}
    \caption{(a): constrained Lasso (partial quantization); (b): constrained Lasso (complete quantization); (c): regularized Lasso (partial quantization); (d): regularized Lasso (complete quantization).}
    \label{fig1}
\end{figure}


\begin{figure}[ht!]
    \centering
    \includegraphics[scale = 0.5]{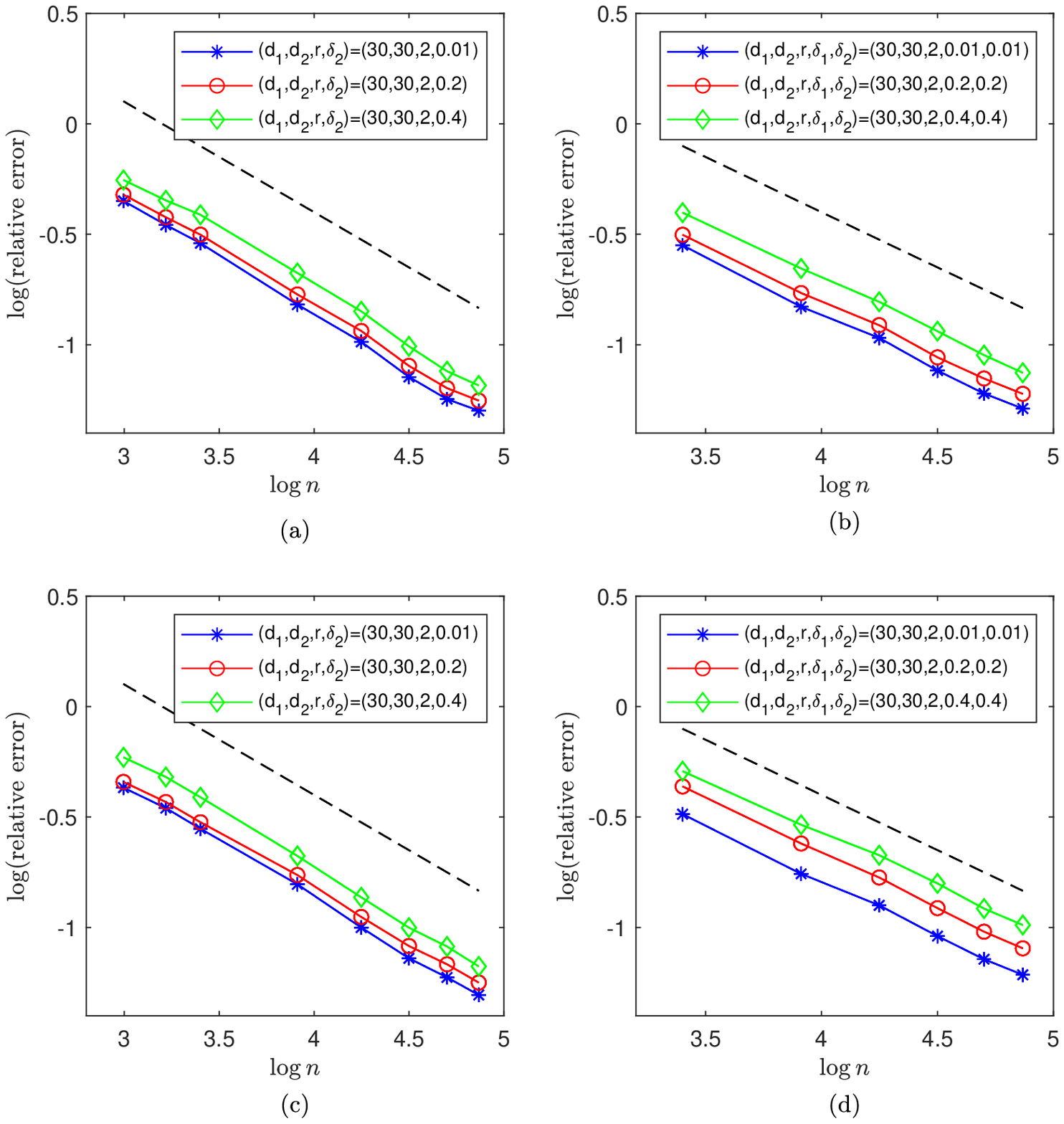}
    \caption{(a): constrained Lasso (partial quantization); (b): constrained Lasso (complete quantization); (c): regularized Lasso (partial quantization); (d): regularized Lasso (complete quantization).}
    \label{fig6}
\end{figure}

\subsubsection{Regularized Lasso for quantized LRMR}
We   switch to the Regularized Lasso estimator, which is more practically appealing in that it does not requires a pre-estimate on $\|\bm{\Theta}\|_{nu}$. The choices of parameters, data generation and quantization are exactly the same as before. We follow the instruction in Theorem \ref{thm2} for choosing $\lambda$ in (\ref{lrmrregu}). That is, for each curve we slightly tune $C(\lambda)$ and then set $\lambda=C(\lambda)\sqrt{\frac{r(d_1+d_2)}{n}}$. We solve the regularized Lasso with ADMM algorithm and show the results in Figure \ref{fig1}(c)-(d).  Note that these results have implications similar to the previous ones for constrained Lasso, in terms of the $O(n^{-1/2})$ decreasing rate, the effect of quantization, problem size, low-rank structure. Thus, we do not repeat the demonstrations.

\textcolor{black}{As suggested by an anonymous reviewer, we simulate quantized LRMR under a sample size $n$ closer or even smaller than $d_1,d_2$. Specifically, we generate the low-rank $\Theta_0\in \mathbb{R}^{30\times 30}$ using the same mechanism, and then test the constrained/regularized Lasso estimators under sample size $n=[20,25,30,50,70,90,110,130]$ for partial quantization, or under $n = [30,50,70,90,110,130]$ for complete quantization.\footnote{We do not test complete quantization under $n<d_1$ because this leads to non-convex program, see (\ref{empi}).} The results in Figure \ref{fig6} indicate that, using sample size close to $d_1$ and $d_2$, the theoretical error bounds   still characterize the estimation errors of our Lasso estimators fairly well.}


\subsubsection{Lasso for quantized L2RM with matrix response}
Now we move to  the problem of low-rank linear regression  with matrix response. Specifically, we 
set $s=4$ in (\ref{matrep}) and  thus there are $\bm{\Theta}_0^{(1)},...,\bm{\Theta}_0^{(4)}$ as underlying coefficients matrices. We generate each   $\bm{\Theta}^{(i)}_0\in\mathbb{R}^{p\times q}$ with rank $\frac{r}{4}$ as before. To fulfill Assumption \ref{assumption2}, we adopt covariates $\bm{x}_k \sim \mathcal{N}(0,\bm{I}_s)$ and the noise matrices $\bm{E}_k\sim \mathcal{N}^{p\times q}(0,0.01)$. We simulate different choices of $(p,q,r,\delta_1,\delta_2)$ under $n=4000:1000:8000$. We note the following facts from the results in Figure \ref{fig2}  that can support our theoretical error rate $O\big(\sqrt{\frac{r(p+q)}{n}}\big)$: all experimental curves decrease with $n$ in a rate of $O(n^{-1/2})$; coarser quantization     only lifts the curve a little bit; larger $(p,q,r)$  results in larger estimation error.

\begin{figure}[ht!]
    \centering
    \includegraphics[scale = 0.34]{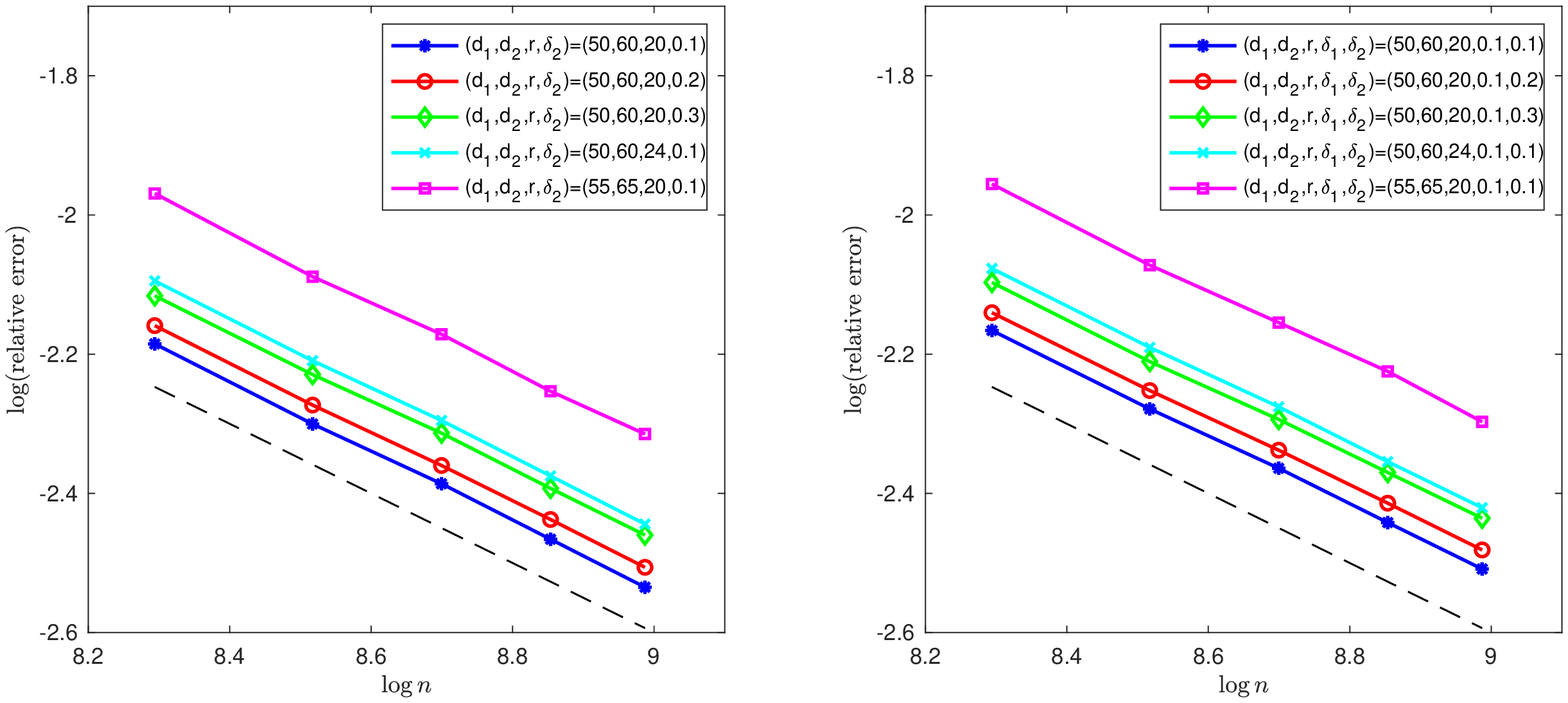}
    \caption{(left): partial quantization; (right):  complete quantization; we denote   $(d_1,d_2)=(p,q)$ in the labels.}
    \label{fig2}
\end{figure}
\subsubsection{The importance of dithering}

As already analysed in section \ref{sec1}, under a direct uniform quantization without dithering, it is in general not possible to estimate the low-rank parameter matrix to arbitrarily small error. To demonstrate this, we use  covariates with entries i.i.d. drawn from $\{\pm 1\}$-valued Bernoulli distribution to simulate LRMR with $50\times 60$ underlying low-rank matrix given by \begin{equation}
    \nonumber
    \bm{\Theta}_0=\begin{bmatrix}
        \bm{A} & 0 \\0 &0
    \end{bmatrix},\bm{B}=\begin{bmatrix}
        0.5 & 0.5 \\ 0.4 & 0.4
    \end{bmatrix},\bm{A}=\mathrm{diag}\big(\underbrace{\bm{B},...,\bm{B}}_{\text{ten}~\bm{B}\text{'s}}\big).
\end{equation}
Also, we simulate (\ref{matrep}) with $s=4$ and $50\times 60$ true matrices \begin{equation}
    \nonumber
    \begin{aligned}
        &\bm{\Theta}_0^{(1)}= \begin{bmatrix}
            \frac{1}{2} \bm{C} &0\\0&0
        \end{bmatrix} ,~\bm{\Theta}_0^{(2)}= \begin{bmatrix}
            \frac{2}{5}\bm{C}&0\\0&0
        \end{bmatrix} ,\\
        &\bm{\Theta}_0^{(3)}=\begin{bmatrix}
            0&0\\0&\frac{1}{2}\bm{C}
        \end{bmatrix},~\bm{\Theta}_0^{(4)}=\begin{bmatrix}
            0 & 0 \\0& \frac{2}{5}\bm{C}
        \end{bmatrix},
    \end{aligned}
\end{equation}
where $\bm{C}=\mathrm{diag}(\bm{D},...,\bm{D})\in \mathbb{R}^{10\times 10}$ has five $\bm{D}$'s, and $\bm{D}$ is the $2\times 2$ all-ones matrix. Under  Gaussian noise, we   quantize the responses with $\delta_2 = 1$ either under the uniform dither $\mathscr{U}[-\frac{1}{2},\frac{1}{2}]$, or directly without dithering. Then we estimate the parameters via regularized Lasso under different sample sizes, the   results are shown in Figure \ref{fig4}. We find that, compared to a direct quantization, using dithering significantly reduces estimation errors; more prominently,   the errors under dithering decrease at a sharp rate, whereas the curves without dithering reach some error floor where more data can no longer improve the estimation. We refer to \cite[Figure 1]{sun2022quantized}, \cite[Figure 5]{chen2022quantizing} for similar experimental results in the contexts of 
compressed sensing, matrix completion, and covariance estimation.

\begin{figure}[ht!]
    \centering
    \includegraphics[scale = 0.63]{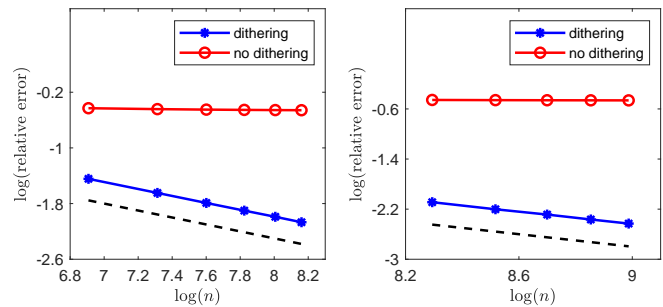}
    \caption{We compare the quantized setting with or without random dithering. (left): quantized LRMR; (right):  quantized L2RM with matrix response.}
    \label{fig4}
\end{figure}
\subsection{Simulations of image restoration}
\textcolor{black}{Note that natural images are approximately low-rank\footnote{This means that  its singular values decrease rapidly and only the first  few are dominant.} (e.g., \cite{chen2022color,chen2019low}), and our theoretical results can   be easily extended to approximately low-rank case by slightly more work (e.g., \cite{negahban2011estimation,fan2021shrinkage,chen2022color}). To better visualize the effect of quantization,  following prior work like \cite{kong2019l2rm},   we conduct simulations with  images as underlying low-rank matrices in this part. } 
\subsubsection{Quantized LRMR}
This numerical example simulates (\ref{3.1}) with each channel of "Peppers" as $\bm{\Theta}_0$, aiming to test the effect of quantization in a relatively high-noise setting. We also demonstrate the advantage of LRMR over the ordinary least squares (OLS) estimation (see Remark \ref{ls}). In the experiment, we separately deal with each channel, which is a $256\times 256$ approximately low-rank matrix (see the left bottom of Figure \ref{fig5}). Specifically, we draw entries of $\bm{x}_k$ from $\mathcal{N}(0,1)$; let $e$ be the average magnitude of the signal part $(\bm{\Theta}_0^\top\bm{x}_k)_{k=1}^n$, we  use $\bm{\epsilon}_k\sim \frac{2e}{5}\cdot \mathcal{N}(0,\bm{I}_{256})$ to simulate a relatively large noise (signal-to-noise ratio less than 7); in the quantized setting, we use uniform dithering and quantize $\bm{y}_k$ with $\delta_2=\frac{e}{8}$. Under $n=300$ or $n=400$, we test regularized Lasso with noisy unquantized/quantized $\bm{y}_k$, as well as OLS with noisy quantized $\bm{y}_k$. The results in Figure \ref{fig5} indicate that, quantization does not notably harm the restoration (comparing columns 2 and 3). Moreover, in such a noisy and quantized setting, Lasso estimator significantly outperforms the OLS estimation that is ignorant of the low-rank structure (comparing columns 3 and 4).
\begin{figure}[ht!]
    \centering
    \includegraphics[scale = 0.35]{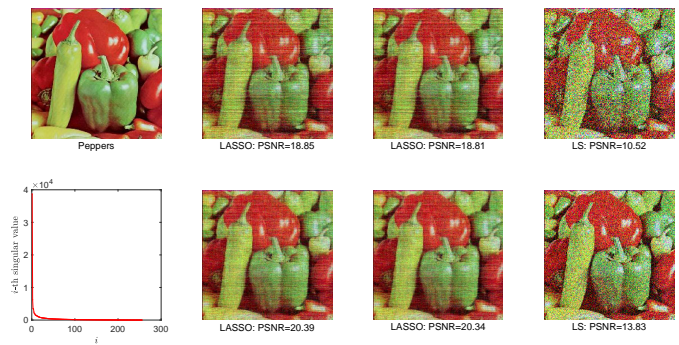}
    \caption{row 1: \underline{$n=300$}; row 2: \underline{$n=400$}. column 1: the true image (peppers) and the singular values of its red channel, (which indicates approximate low-rankness); column 2: Lasso with noisy unquantized data; column 3: Lasso with noisy quantized data; column 4: Least squares estimation with noisy quantized data.}
    \label{fig5}
\end{figure}
\subsubsection{Quantized L2RM with matrix response}
    We follow the experiment in \cite[Figures 1-2]{kong2019l2rm}. Specifically, we simulate (\ref{matrep}) with $s=4$ where $\bm{\Theta}_0^{(i)}$'s are $64\times 64$ $0$-$1$ matrices and shown as images in the first row of Figure \ref{fig3}. It can be verified that they are approximately low-rank. We also adopt the method of generating $(\bm{x}_k,\bm{E}_k)$ in \cite{kong2019l2rm}. While   the experiment in \cite{kong2019l2rm} aims at comparing different methods of recovering $\bm{\Theta}_0^{(i)}$, however, our main goal here is to exhibit how quantization resolution affects the recovery. Thus, we simulate the regularized Lasso (\ref{regulasso}) under   response quantization with  $\delta_2=0.0,0.5,1.0,3.0$. Under the sample size of $n=2000$, the reconstructed images are shown in   rows two through five in Figure \ref{fig3}. We also run 100 independent trials and report the mean (relative) Frobenius norm error   and standard deviation for each $\bm{\Theta}_0^{(i)}$ in Table \ref{tab2}. It is clear both visually and on the mean error that, under quantization with relatively high resolution ($\delta_2=0.5,1.0$),  Lasso returns estimations fairly close to the ones obtained in a full-data regime. In fact, even if we quantize $\bm{Y}_k$ with $\delta_2=3$,\footnote{This represents rather low resolution because in the simulation, entries of $\bm{Y}_k$ have magnitude about $1.6$ in average.} the Lasso estimator still delivers quite acceptable results. Therefore, we conclude that the dithered quantization will not significantly deteriorate one's ability to recover the underlying low-rank parameters; rather, the dithered uniform quantizer preserves the information fairly well. Generally speaking, there should be a trade-off between quantization resolution and recovery accuracy in practice. Note that the smaller sample size $n=400$ is  also simulated, see Table \ref{tab1} for the results with similar implications.

\begin{figure}[ht!]
    \centering
    \includegraphics[scale = 0.7]{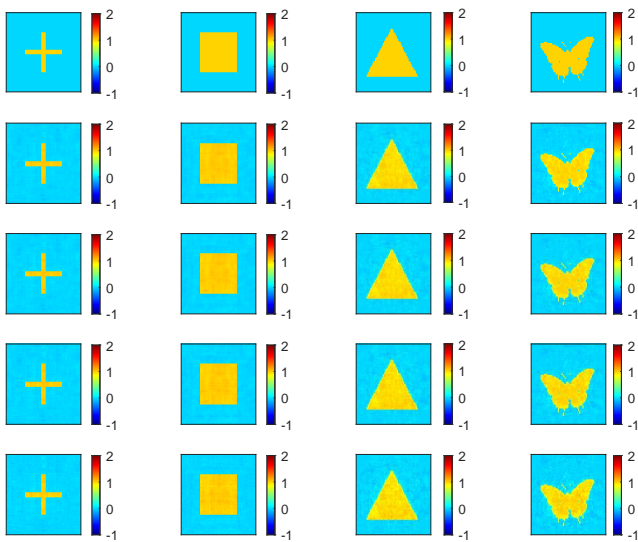}
    \caption{\underline{$n=2000$}. (row 1): the  true images $\bm{\Theta}_0^{(i)}$; (row 2-5): images reconstructed by regularized Lasso (\ref{regulasso}) under response quantization with $\delta_2=0.0$ (no quantization), $\delta_2=0.5$, $\delta_2 = 1.0$ and $\delta_2=3.0$, respectively.}
    \label{fig3}
\end{figure}

\begin{table}[H]
\begin{center}
\begin{tabular}{c|c|c|c|c}
\toprule
$\delta_2$& $\bm{\Theta}_0^{(1)}$ &$\bm{\Theta}_0^{(2)}$& $\bm{\Theta}_0^{(3)}$&$\bm{\Theta}_0^{(4)}$\\
 \midrule
0.0&0.0532(7.91)&0.0263(5.53)&	0.0791(4.71)&0.0801(5.90)\\
0.5&0.0536(7.86)&0.0265(5.48)&	0.0795(4.67)&0.0804(5.86)\\
1.0&0.0551(7.67)&0.0272(5.39)&0.0805(4.62)&0.0813(5.77)
\\
3.0&0.0769(5.96)&0.0376(4.23)	&0.0920(4.12)&0.0923(5.00)\\
\bottomrule
\end{tabular}
\end{center}
\caption{\underline{$n=2000$}. Mean relative Frobenius norm errors (standard deviation($\times 10^{-3}$))}
\label{tab2}
\end{table}

\begin{table}[H]
\begin{center}
\begin{tabular}{c|c|c|c|c}
\toprule
$\delta_2$& $\bm{\Theta}_0^{(1)}$ &$\bm{\Theta}_0^{(2)}$& $\bm{\Theta}_0^{(3)}$&$\bm{\Theta}_0^{(4)}$\\
 \midrule
0.0&0.1470(8.60)&	0.0775(9.57)&	0.1233(10.29)&	0.1249(10.50)\\
0.5&0.1496(8.34)&	0.0787(9.48)&	0.1245(10.22)&	0.1261(10.36)\\
1.0&0.1575(8.08)&	0.0825(9.09)&	0.1284(9.81)&	0.1300(10.09)\\
3.0&0.2507(6.05)&	0.1278(6.30)&	0.1699(7.67)&	0.1723(7.36)\\
\bottomrule
\end{tabular}
\end{center}
\caption{\underline{$n=400$}. Mean relative Frobenius norm errors (standard deviation($\times 10^{-3}$))}
\label{tab1}
\end{table}

\subsection{\textcolor{black}{A real data application}}

\textcolor{black}{
To confirm the efficacy of the  proposed method, we perform the quantization and estimation in a 
genetic association study for examining the regulatory control mechanisms in gene networks for isoprenoids in Arabidopsis thaliana \cite{she2017robust,wille2004sparse}. We adopt the LRMR model (\ref{3.1}) with $\bm{x}_k$ being the expression levels of $d_1=39$ genes from the two isoprenoid biosynthesis pathways, $\bm{y}_k$ being the expression levels of $d_2=62$ genes from four downstream pathways, and we use $n=115$ samples in total.\footnote{There are originally  $118$   samples in this real data study, but we remove $3$ samples that are detected as (potential) outliers in \cite{she2017robust}, see Figure 1 therein.} Besides, the mean magnitudes of the entries of $\bm{x}_k$ and $\bm{y}_k$ are $2160$ and $3707$, respectively. }

\textcolor{black}{
 We will focus on how the dithered quantization of $(\bm{x}_k,\bm{y}_k)$ affects the estimation and prediction of regularized Lasso (\ref{3.3}). 
  Note that the two major differences between this real data application and the previous simulations are that the data here may not be nicely captured by the sub-Gaussian distributions (Assumption \ref{assumpt1}), and that the relation between $\bm{x}_k,\bm{y}_k$ may not be perfectly modeled by LRMR (\ref{3.1}).  Thus, there is not an underlying $\bm{\Theta}_0$ serving as the ground truth. Alternatively, since the emphasis is on the effect of quantization, we use the Lasso estimator with suitable $\lambda$
  from unquantized data as $\bm{\Theta}_0$.}
  
  \textcolor{black}{
  For partial quantization, we quantize $\bm{y}_k$ to $\bm{\dot{y}}_k$ under $\delta_2=0:100:1000$ and obtain $\bm{\Theta}_\delta$ from $(\bm{x}_k,\bm{\dot{y}}_k)$ as in (\ref{lrmrregu}), where the parameter $\lambda$ increases with $\delta_2$, as instructed by Theorem \ref{thm2}. The relative estimation error $\frac{\|\bm{\Theta}_\delta-\bm{\Theta}_0\|_F}{\|\bm{\Theta}_0\|_F}$ and relative prediction error $\frac{\|\bm{Y}-\bm{\Theta}_\delta^\top \bm{X}\|_F}{\|\bm{Y}\|_F}$ are reported as their mean values in 50 independent trials in Figure \ref{fig10} (a)-(b). Specifically, the curves slowly increase with $\delta_2$; compared to the unquantized case $\delta_2=0$, the estimation and prediction under the coarse quantization $\delta_2=1000$ are still acceptable. We also test the complete quantization setting where $\bm{x}_k$ is quantized to $\bm{\dot{x}}_k$ with $\delta_1=0:5:50$, $\bm{y}_k$ is quantized to $\bm{\dot{y}}_k$ with $\delta_2=0:50:500$. Similar results are reported in Figure \ref{fig10} (c)-(d), but comparing Figure \ref{fig10}(c) and Figure \ref{fig10}(a),
  we also note that $\bm{\Theta}_\delta$ deviates from $\bm{\Theta}_0$ more significantly  in complete quantization (even though $\delta_1=0:5:50$ is relatively small compared to the mean magnitude of $\bm{x}_k$); that is, the quantization of $\bm{x}_k$ affects the estimation more severely. Finally, we conduct a more practical learning and prediction setting as follows:   randomly dividing the columns  of $\bm{X}\in \mathbb{R}^{39\times 115},\bm{Y}\in \mathbb{R}^{62\times 115}$ into the "training data" $\bm{X}_1 \in \mathbb{R}^{39\times 95},\bm{Y}_1\in \mathbb{R}^{62\times 95}$ and the "testing data" $\bm{X}_2 \in \mathbb{R}^{39\times 20},\bm{Y}_2\in \mathbb{R}^{62\times 20}$, we quantize $\bm{Y}_1$ to $\bm{\dot{Y}}_1$ with $\delta_2= 0:100:1000$ and use $(\bm{X}_1,\bm{\dot{Y}}_1)$ to obtain the estimator $\bm{\Theta}_\delta$ defined in (\ref{lrmrregu}),  then we track the relative prediction error over the testing data, i.e., $\frac{\|\bm{Y}_2-\bm{\Theta}_\delta^\top\bm{X}_2\|_F}{\|\bm{Y}_2\|_F}$, whose mean value in $50$ independent trials is reported in Figure \ref{fig10}(e). Compared to Figure \ref{fig10}(b), the prediction error increases even more slowly with $\delta_2$. In conclusion, our quantization scheme well preserves the data information for subsequent estimation and prediction procedures.}

\begin{figure*}[!t]
    \centering
    \includegraphics[scale = 0.66]{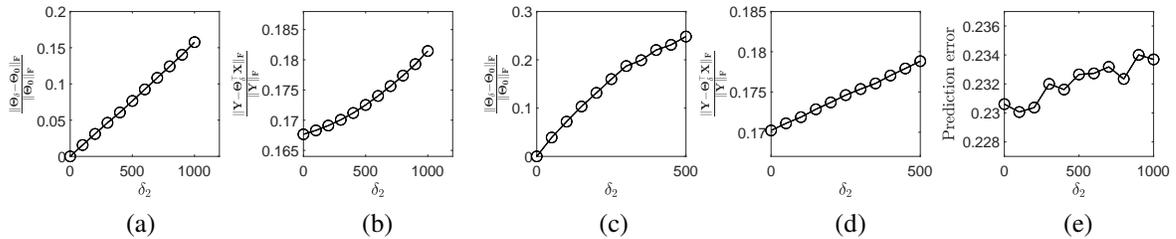}

     ~~~~~~~~~~~~~~   (a) \hspace{2.5cm} (b)  \hspace{2.5cm} (c) \hspace{2.5
     cm} (d) \hspace{2.5cm}(e)
    \caption{(a)-(b): estimation/prediction errors, partial quantization; (c)-(d): estimation/prediction errors, complete quantization; (e): prediction error over "testing data", partial quantization.}
    \label{fig10}
\end{figure*}

 \section{Conclusions}\label{sec6}
This paper, for the first time, studied low-rank multivariate regression (LRMR) in a realistic setting that involves data quantization. We proposed to use the dithered uniform quantizer, associated with uniform dither for the response, or with triangular dither for the covariate. We proposed the Lasso estimators based on quantized data in a constrained or regularized manner. With the aid of random dithering, albeit losing information in quantization, our estimators achieve minimax optimal error rate. In fact, the derived error bounds demonstrate that the quantization only results in slightly worse multiplicative factors, which is reminiscent of similar results in quantized CS (Remark \ref{rem2}) and has been clearly observed in our simulations (e.g., Figure \ref{fig1}). Moreover, we similarly applied the quantization scheme to a low-rank regression problem with matrix response and established the theoretical results accordingly. Experimental results were reported to complement our theoretical developments.

For future work, our first direction is to study LRMR under the more extreme 1-bit quantization, which only retains the sign of the data. Secondly, while we separately worked on LRMR and L2RM with matrix response in this paper, it would be of interest to attempt to unify their analyses, and ideally build a general theoretical framework for quantized multiresponse regression. \textcolor{black}{Last but not least, it is desired to investigate whether our quantization method and theoretical results could be extended to a high-dimensional setting where $n<d_1$, which probably requires  new machinery in the technical proofs and  structure on $\bm{\Theta}_0$ beyond low-rankness.}

 \bibliographystyle{ieeetran}
\bibliography{libr}
\begin{appendix}

   \subsection{The proof of Corollary \ref{coro1}}
   \begin{proof}
 Following   Lemma \ref{lem1}, the proof can be done by some elementary algebraic manipulation. For the first part of the claim, we only need to verify both choices of $\tau_i$ satisfy the condition in Lemma 1(a): If $\tau_i\sim \mathscr{U}\big([-\frac{\delta}{2},\frac{\delta}{2}]\big)$, then  \begin{equation}
    \begin{aligned}\nonumber
     \mathbbm{E} \exp(\textbf{i}u\tau_i) &= \int_{-\frac{\delta}{2}}^{\frac{\delta}{2}} \frac{1}{\delta}\big(\cos (ux) + \textbf{i}\sin (ux)\big) \mathrm{d}x \\&= \frac{2}{\delta u}\sin\big(\frac{\delta u}{2}\big),
    \end{aligned}
\end{equation}
which obviously vanishes at $u= \frac{2\pi l}{\delta}$ for non-zero integer $l$; It is similar for triangular dither. For the second part of the claim, let us show the triangular dither satisfies the condition in Lemma 1(b). Let $Z \sim \mathscr{U}\big([-\frac{\delta}{2},\frac{\delta}{2}]\big)$ be independent of $\tau_i$, then \begin{equation}
    \begin{aligned}\nonumber
    &g(u):= \mathbbm{E} (\exp(\textbf{i}uZ))\mathbbm{E} (\exp(\textbf{i}u\tau_i))\\&= \big[\mathbbm{E} (\exp(\textbf{i}uZ))\big]^3 = \Big[\frac{2}{\delta u} \sin\big(\frac{\delta u}{2}\big)\Big]^3.
    \end{aligned}
\end{equation}
It is evident that $g''(u)$ contains a common factor $\sin\big(\frac{\delta u}{2}\big)$, thus $g''(\frac{2\pi l}{\delta })=0$ holds for any non-zero integer $l$. Hence, $$\mathbbm{E}\xi_i^2=\mathbbm{E}(Z+\tau_i)^2= \mathbbm{E}Z^2+\mathbbm{E}\tau_i^2 =3\int_{-\delta/2}^{\delta/2}~\frac{x^2}{\delta}~\mathrm{d}x=\frac{\delta^3}{4},$$ the proof is complete. 
\end{proof}
\subsection{The proof of Theorem \ref{thm2}}
\begin{proof}
We start with the optimality of $\bm{\widehat{\Theta}}_p$
$$\dot{\mathcal{L}}(\bm{\widehat{\Theta}}_p)+\lambda\|\bm{\widehat{\Theta}}_p\|_{nu}\leq \dot{\mathcal{L}}(\bm{{\Theta}}_0)+\lambda\|\bm{{\Theta}}_0\|_{nu}.$$
Recall that $\bm{\widehat{\Delta}}_p=\bm{\widehat{\Theta}}_p-\bm{\Theta}_0$, by some algebra we arrive at \begin{equation}\label{A.3}
\begin{aligned}    \big<\bm{\widehat{\Delta}}_p\bm{\widehat{\Delta}}_p^\top,\bm{\widehat{\Sigma}_{xx}}\big> \leq & 2\big<\bm{\widehat{\Delta}}_p,\bm{\widehat{\Sigma}_{xy}}-\bm{\widehat{\Sigma}_{xx}}\bm{\Theta}_0\big> \\&+ \lambda \big(\|\bm{\Theta}_0\|_{nu}-\|\bm{\widehat{\Theta}}_p\|_{nu}\big).\end{aligned}
\end{equation} 
Note that the left-hand side is always non-negative (this holds deterministically when $\delta_1=0$, and holds within the promised probability when $\delta_1>0$, see {\it step 1} in the proof of Theorem \ref{theorem1}). By (\ref{3.16}), (\ref{add1}), (\ref{3.19}) in the proof of Theorem \ref{theorem1}, in both "partial quantization" and "complete quantization", our choices of $\lambda$ can guarantee $\lambda \geq 4 \|\bm{\widehat{\Sigma}_{xy}}-\bm{\widehat{\Sigma}_{xx}\Theta}_0\|_{op}$ holds under the promised probability. Under the same probability, we thus obtain $$0\leq \frac{\lambda}{2}\|\bm{\widehat{\Delta}}_p\|_{nu}+\lambda\big(\|\bm{\Theta}_0\|_{nu}-\|\bm{\widehat{\Theta}}_p\|_{nu}\big),$$ i.e., $\frac{1}{2}\|\bm{\widehat{\Delta}}_p\|_{nu}+\|\bm{\Theta}_0\|_{nu}\geq \|\bm{\widehat{\Theta}}_p\|_{nu}$. Recall that (\ref{3.15}) can provide $\|\bm{\widehat{\Theta}}_p\|_{nu}-\|\bm{\Theta}_0\|_{nu}\geq \|\mathcal{P}_{\overline{\mathcal{M}}^\bot}\bm{\widehat{\Delta}}_p\|_{nu}-\|\mathcal{P}_{\overline{\mathcal{M}}}\bm{\widehat{\Delta}}_p\|_{nu}$, we deduce $$\big\|\mathcal{P}_{\overline{\mathcal{M}}^\bot} \bm{\widehat{\Delta}}_p\big\|_{nu}\leq \big\|\mathcal{P}_{\overline{\mathcal{M}}}\bm{\widehat{\Delta}}_p\big\|_{nu}+\frac{1}{2}\big\|\bm{\widehat{\Delta}}_p\big\|_{nu},$$
where the involved subspaces and projections are defined in the proof of Theorem \ref{theorem1}. Thus, \begin{equation}\nonumber\begin{aligned}
    &\big\|\bm{\widehat{\Delta}}_p\big\|_{nu}\leq \big\|\mathcal{P}_{\overline{\mathcal{M}}}\bm{\widehat{\Delta}}_p\big\|_{nu}+\big\|\mathcal{P}_{\overline{\mathcal{M}}^\bot}\bm{\widehat{\Delta}}_p\big\|_{nu}\\&\leq 2\big\|\mathcal{P}_{\overline{\mathcal{M}}}\bm{\widehat{\Delta}}_p\big\|_{nu}+\frac{1}{2}\big\|\bm{\widehat{\Delta}}_p\big\|_{nu},\end{aligned}
\end{equation} 
which further gives $$\|\bm{\widehat{\Delta}}_p\|_{nu}\leq 4\|\mathcal{P}_{\overline{\mathcal{M}}}\bm{\widehat{\Delta}}_p\|_{nu}\leq 4\sqrt{2r}\|\bm{\widehat{\Delta}}_p\|_F,$$ and the last inequality is because $\rank(\bm{A})\leq 2r$ if $\bm{A}\in \overline{\mathcal{M}}$.

Having deduced $\|\bm{\widehat{\Delta}}_p\|_{nu}\lesssim \sqrt{r} \|\bm{\widehat{\Delta}}_p\|_F$, we can upper bound the right-hand side of (\ref{A.3}) by 
\begin{equation}\nonumber
    \begin{aligned}
        &\frac{\lambda}{2}\|\bm{\widehat{\Delta}}_p\|_{nu} + \lambda\big(\|\bm{\Theta}_0\|_{nu}-\|\bm{\widehat{\Theta}}_p\|_{nu}\big)\\&~~~~~~\leq \frac{3}{2}\lambda\|\bm{\widehat{\Delta}}_p\|_{nu}\lesssim \lambda\sqrt{r}\|\bm{\widehat{\Delta}}_p\|_F.
    \end{aligned}
\end{equation}
As deduced in (\ref{3.14}), within the promised probability the left-hand side of (\ref{A.3}) is lower bounded by $\frac{\kappa_0}{2}\|\bm{\widehat{\Delta}}_p\|_F^2$. 
Thus, we arrive at $\|\bm{\widehat{\Delta}}_p\|_F\lesssim \kappa_0^{-1}\lambda\sqrt{r}$. To complete the proof, we only need to plug in the value of $\lambda$ in both cases.
\end{proof}
\subsection{The proof of Theorem \ref{theorem3}}
\begin{proof}
    We define $\bm{\Delta}=\bm{\widehat{\Theta}}- \bm{\Theta}_0 \in \mathbb{R}^{p\times sq}$, $\bm{\Delta}^{(i)}=\bm{\widehat{\Theta}}^{(i)}-\bm{\Theta}_0^{(i)}\in \mathbb{R}^{p\times q}$, recall the rearrangement   $\bm{\widetilde{\Theta}}\in \mathbb{R}^{s\times pq}$ defined in (\ref{rearr}). We continue to use prior notation for  quantization noise/error: $$\bm{\dot{x}}_k=\bm{x}_k+\bm{\xi}_{k1}=\bm{x}_k+\bm{\phi}_k+\bm{w}_{k1}$$and $$\bm{\dot{Y}}_k=\bm{Y}_k+\bm{\xi}_{k2}=\bm{Y}_k+\bm{\tau}_k+\bm{w}_{k2}.$$ Now we use the definition and obtain $\dot{\mathcal{L}}_1(\bm{\widehat{\Theta}})+ \lambda \cdot \sum_{i=1}^s\|\bm{\widehat{\Theta}}\|_{nu}\leq \dot{\mathcal{L}}_1(\bm{{\Theta}}_0)+ \lambda \cdot \sum_{i=1}^s\|\bm{{\Theta}}_0\|_{nu}$. Then we perform some algebra to arrive at \begin{equation}
\label{key}\begin{aligned}&\big<\bm{\widetilde{\Delta}}\bm{\widetilde{\Delta}}^\top,\bm{\widehat{\Sigma}}_{\bm{xx}}\big>\leq 2\big<\bm{\widehat{\Sigma}}_{\bm{xy}}-\bm{\widehat{\Sigma}}_{\bm{xx}}\bm{\widetilde{\Theta}}_0,\bm{\widetilde{\Delta}}\big>\\&~~~~+\lambda\cdot \sum_{i=1}^s\big(\|\bm{\Theta}_0^{(i)}\|_{nu}-\|\bm{\widehat{\Theta}}^{(i)}\|_{nu}\big).
        \end{aligned}
    \end{equation}

    \noindent{\it Step 1. Bound the left-hand side from below.} 

    This is exactly the same as {\it Step 1} in the proof of Theorem \ref{theorem1}. In more detail, because $n\gtrsim s$, one can invoke Lemma \ref{lemma4} to show that $\lambda_{\min}(\bm{\widehat{\Sigma}_{xx}})\geq \frac{\kappa_0}{2}$ holds with probability at least $1-2\exp(-s)$. Assume that we are on this event, then evidently we have $$\big<\bm{\widetilde{\Delta}}\bm{\widetilde{\Delta}}^\top,\bm{\widehat{\Sigma}_{xx}}\big>\geq \frac{\kappa_0}{2}\|\bm{\widetilde{\Delta}}\|_F^2=\frac{\kappa_0}{2}\|\bm{\Delta}\|_F^2.$$  

    \vspace{1mm}
    \noindent{\it Step 2. Bound   $\mathscr{T}:=\big<\bm{\widehat{\Sigma}}_{\bm{xy}}-\bm{\widehat{\Sigma}}_{\bm{xx}}\bm{\widetilde{\Theta}}_0,\bm{\widetilde{\Delta}}\big>$ from above.}

    Using $\mathrm{vec}(\bm{Y}_k)=\bm{\widetilde{\Theta}}_0^\top \bm{x}_k+\mathrm{vec}(\bm{E}_k)$ and the meaning of $\bm{\xi}_{k1},\bm{\xi}_{k2}$, we can first simplify $\bm{\widehat{\Sigma}_{xy}}-\bm{\widehat{\Sigma}_{xx}}\bm{\widetilde{\Theta}}_0$ to \begin{equation}
        \begin{aligned}\nonumber
            &\underbrace{\frac{1}{n}\sum_{k=1}^n \bm{\dot{x}}_k \big(\mathrm{vec}(\bm{E}_k)+\mathrm{vec}(\bm{\xi}_{k2})\big)^\top}_{\mathscr{T}_1} -\underbrace{\frac{1}{n}\sum_{k=1}^n\big(\bm{\dot{x}}_k\bm{\xi}_{k1}^\top  - \frac{\delta_1^2}{4}\bm{I}_s\big)\bm{\widetilde{\Theta}}_0}_{\mathscr{T}_2}.
        \end{aligned}
    \end{equation}
     Thus, we have $|\mathscr{T}|\leq \big|\big<\mathscr{T}_1,\bm{\widetilde{\Delta}}\big>\big|+\big|\big<\mathscr{T}_2,\bm{\widetilde{\Delta}}\big>\big|$, and it amounts to estimating $\big|\big<\mathscr{T}_1,\bm{\widetilde{\Delta}}\big>\big|$ and $\big|\big<\mathscr{T}_2,\bm{\widetilde{\Delta}}\big>\big|$. For the first term,  by turning back to the $\mathbb{R}^{p\times q}$ we have \begin{equation}
        \begin{aligned}\label{reviseadd3}            &\big|\big<\mathscr{T}_1,\bm{\widetilde{\Delta}}\big>\big|=\Big| \Big<\frac{1}{n}\sum_{k=1}^n\bm{\dot{x}}_k\big(\mathrm{vec}(\bm{E}_k)+\mathrm{vec}(\bm{\xi}_{k2})\big)^\top,\bm{\widetilde{\Delta}}\Big>\Big|\\&= \Big|\sum_{i=1}^s\Big<\frac{1}{n}\sum_{k=1}^n\dot{x}_{ki}(\bm{E}_k+\bm{\xi}_{k2}), \bm{\Delta}^{(i)}\Big>\Big|\\&\leq \Big(\sum_{i=1}^s \|\bm{\Delta}^{(i)}\|_{nu}\Big)\Big(\max_{i\in [s]}\Big\|\frac{1}{n}\sum_{k=1}^n\dot{x}_{ki}(\bm{E}_k+\bm{\xi}_{k2})\Big\|_{op}\Big)\\&\leq C(K+\delta_1)(E+\delta_2)\sqrt{\frac{p+q}{n}}\cdot \Big(\sum_{i=1}^s\|\bm{\Delta}^{(i)}\|_{nu}\Big),
        \end{aligned}
    \end{equation}where in the last inequality we invoke Lemma \ref{lemma5} and a union bound over $i\in [s]$; it holds with probability at least $1-2\exp(-c(p+q))$ because $\log s =O(p+q)$. Note that the second term $\big|\big<\mathscr{T}_2,\bm{\widetilde{\Delta}}\big>\big|$ vanishes in partial quantization ($\delta_1=0$), thus we estimate it on the complete quantization case $(\delta_1>0)$ where we further assume $\sum_i \|\bm{\Theta}_0^{(i)}\|^2_{op}\leq R^2$ and $s=O(p+q)$. In particular, we define $$\bm{\Psi}=[\psi_{ij}]=\frac{1}{n}\sum_{k=1}^n\big(\bm{\dot{x}}_k\bm{\xi}_{k1}^\top- \frac{\delta_1^2}{4}\bm{I}_s\big)\in \mathbb{R}^{s\times s},$$ and note that we have $\mathbbm{E}\bm{\Psi}=0$. Moreover Lemma \ref{lemma5} provides that, $\|\bm{\Psi}\|_{op}\lesssim (K+\delta_1)\delta_1\sqrt{\frac{s}{n}}$ holds with probability at least $1-\exp(-s)$. On this event, we estimate that
    
\begin{equation}
    \begin{aligned}
        \label{reviseadd1}
&\big|\big<\mathscr{T}_2,\bm{\widetilde{\Delta}}\big>\big|= \big|\big<\bm{\Psi}\bm{\widetilde{\Theta}}_0,\bm{\widetilde{\Delta}}\big>\big|\\&=\Big|\sum_{i=1}^s \big<\sum_{j=1}^s \psi_{ij}\bm{\Theta}_0^{(j)},\bm{\Delta}^{(i)}\big>\Big|\\&\leq \Big(\max_{i\in [s]}\Big\|\sum_{j=1}^s\psi_{ij}\bm{\Theta}_0^{(j)}\Big\|_{op}\Big)\Big(\sum_{i=1}^s\|\bm{\Delta}^{(i)}\|_{nu}\Big)\\&\leq C(K+\delta_1)R\delta_1 \sqrt{\frac{p+q}{n}}\Big(\sum_{i=1}^s\|\bm{\Delta}^{(i)}\|_{nu}\Big),
    \end{aligned}
\end{equation}
    where the last inequality is because for $i\in [s]$, \begin{equation}\label{reviseadd2}
        \begin{aligned}
            &\|\sum_j \psi_{ij}\bm{\Theta}_0^{(j)}\|_{op}\leq \sum_j|\psi_{ij}|\|\bm{\Theta}_0^{(j)}\|_{op}\\&\leq \big(\sum_j\psi_{ij}^2\big)^{1/2}\big(\sum_j\|\bm{\Theta}^{(j)}_0\|_{op}^2\big)^{1/2}\\&\leq R\|\bm{\Psi}\|_{op}=O\big((K+\delta_1)R\delta_1\sqrt{\frac{s}{n}}\big)
        \end{aligned}
    \end{equation} 
     also recall that $s=O(p+q)$.

    To conclude, in "partial quantization" we have shown $\mathscr{T}= O\big(A_6\sqrt{\frac{p+q}{n}}\big)$, and in "complete quantization" $\mathscr{T}=O\big(A_7\sqrt{\frac{p+q}{n}}\big)$. Compared to  our choices of $\lambda$, we can assume $2\mathscr{T}\leq \frac{1}{2}\lambda \big(\sum_i\|\bm{\Delta}^{(i)}\|_{nu}\big)$ with the promised probability. Because the left-hand side of (\ref{key}) is non-negative (deterministically if $\delta_1=0$, with the promised probability if $\delta_1>0$), and $\lambda>0$, we arrive at \begin{equation}
       \label{here} \sum_{i=1}^s\big(\|\bm{\widehat{\Theta}}^{(i)}\|_{nu}-\|\bm{\Theta}^{(i)}_0\|_{nu}\big)\leq \frac{1}{2}\sum_{i=1}^s \|\bm{\Delta}^{(i)}\|_{nu}.
    \end{equation}

    \noindent{\it Step 3. Conclude the proof.}

    We use a decomposability argument. In particular, we let $r_i=\rank(\bm{\Theta}_0^{(i)})$, and exactly the same as the definition of $(\mathcal{M},\overline{\mathcal{M}},\overline{\mathcal{M}}^\bot)$   at the beginning of {\it Step 2} in the proof of Theorem \ref{theorem1} (regarding $\bm{\Theta}_0$ thereof), we now define $(\mathcal{M}_i,\overline{\mathcal{M}}_i,\overline{\mathcal{M}}_i^\bot)$ regarding $\bm{\Theta}_0^{(i)}$. Similarly, we have the decomposability $$\|\mathcal{P}_{\mathcal{M}_i}\bm{A}+\mathcal{P}_{\overline{\mathcal{M}}_i^\bot}\bm{B}\|_{nu}=\|\mathcal{P}_{\mathcal{M}_i}\bm{A}\|_{nu}+\|\mathcal{P}_{\overline{\mathcal{M}}_i^\bot}\bm{B}\|_{nu}$$ holds for all $i\in [s]$ and $\bm{A},\bm{B}\in \mathbb{R}^{p\times q}$. Thus, we can use (\ref{3.15}) to obtain $$\|\bm{\widehat{\Theta}}^{(i)}\|_{nu}-\|\bm{\Theta}_0^{(i)}\|_{nu}\geq \big\|\mathcal{P}_{\overline{\mathcal{M}}_i^\bot}\bm{\Delta}^{(i)}\big\|_{nu}-\big\|\mathcal{P}_{\overline{\mathcal{M}}_i}\bm{\Delta}^{(i)}\big\|_{nu}.$$
    Putting this into the left-hand side of (\ref{here}), and also apply $$\|\bm{\Delta}^{(i)}\|_{nu}\leq \big\|\mathcal{P}_{\overline{\mathcal{M}}_i^\bot}\bm{\Delta}^{(i)}\big\|_{nu}+\big\|\mathcal{P}_{\overline{\mathcal{M}}_i}\bm{\Delta}^{(i)}\big\|_{nu}$$ to the right-hand side, it provides $\sum_i\big\|\mathcal{P}_{\overline{\mathcal{M}}_i^\bot}\bm{\Delta}^{(i)}\big\|_{nu}\leq 3\sum_i\big\|\mathcal{P}_{\overline{\mathcal{M}}_i}\bm{\Delta}^{(i)}\big\|_{nu}$, which leads to \begin{equation}\nonumber
        \begin{aligned}
            &\sum_i\|\bm{\Delta}^{(i)}\|_{nu}\leq 4\sum_i\big\|\mathcal{P}_{\overline{\mathcal{M}}_i}\bm{\Delta}^{(i)}\big\|_{nu}\leq 4\sum_i \sqrt{2r_i}\|\bm{\Delta}^{(i)}\|_F \\&\leq4\sqrt{2}\big({\sum _ir_i}\big)^{1/2}\big({\sum_i\|\bm{\Delta}^{(i)}\|_F^2}\big)^{1/2}\leq 4\sqrt{2r}\cdot\|\bm{\Delta}\|_F.  
        \end{aligned}
    \end{equation} 
Now we are ready to put pieces together. Because $\|\bm{\Theta}_0^{(i)}\|_{nu}-\|\bm{\widehat{\Theta}}^{(i)}\|_{nu}\leq \|\bm{\Delta}^{(i)}\|_{nu}$, overall, the right-hand side of (\ref{key}) has the bound $O\big(\lambda \sum_i\|\bm{\Delta}^{(i)}\|_{nu}\big)=O\big(\sqrt{r}\lambda\|\bm{\Delta}\|_F\big)$, while the left-hand side is lower bounded by $\frac{\kappa_0}{2}\|\bm{\Delta}\|_F^2$, so it holds with the promised probability that, $\|\bm{\Delta}\|_F=O\big(\frac{\sqrt{r}\lambda}{\kappa_0}\big)$. The proof can be concluded by using the chosen value of $\lambda$.     
\end{proof}
  \subsection{Auxiliary facts}
    \subsubsection{The proof of Lemma \ref{lem3}}
    \begin{proof}
    The proof is a standard covering argument for controlling the matrix operator norm. We construct $\mathcal{N}_1\subset \mathbb{S}^{d_1-1}$ as a $\frac{1}{4}$-net of $\mathbb{S}^{d_1-1}$, meaning that for any $\bm{v}\in \mathbb{S}^{d_1-1}$ there exists $\bm{x}\in\mathcal{N}_1$ such that $\|\bm{x}-\bm{v}\|_2\leq \frac{1}{4}$. Similarly, let $\mathcal{N}_2$ be a $\frac{1}{4}$-net of $\mathbb{S}^{d_2-1}$. By \cite[Corollary 4.2.13]{vershynin2018high} we can assume $|\mathcal{N}_1|\leq 9^{d_1}$, $|\mathcal{N}_2|\leq 9^{d_2}$. Note that for any $\bm{u}\in \mathcal{N}_1$, $\bm{v}\in \mathcal{N}_2$, we have \begin{equation}
        \begin{aligned}\label{center}
            &\big\|\bm{u}^\top \bm{a}_k\bm{b}_k^\top\bm{v} -\mathbbm{E}(\bm{u}^\top \bm{a}_k\bm{b}_k^\top\bm{v})\big\|_{\psi_1} \\&\stackrel{(i)}{\leq}C_1\big\|(\bm{u}^\top\bm{a}_k)(\bm{v}^\top\bm{b}_k)\big\|_{\psi_1}\\&\stackrel{(ii)}{\leq}C_1\|\bm{u}^\top\bm{a}_k\|_{\psi_2}\|\bm{v}^\top\bm{b}_k\|_{\psi_2}\leq C_1E_1E_2
        \end{aligned}
    \end{equation}
    Note that $(i)$ is due to centering \cite[Exercise 2.7.10]{vershynin2018high}, and we use (\ref{2.3})  in $(ii)$. Thus, we can use Bernstein's inequality (see \cite[Theorem 2.8.1]{vershynin2018high}) to obtain the concentration of $\frac{1}{n}\sum_k\big(\bm{u}^\top\bm{a}_k\bm{b}_k^\top\bm{v}-\mathbbm{E}(\bm{u}^\top\bm{a}_k\bm{b}_k^\top\bm{v})\big)$; Followed by a union bound over $(\bm{u},\bm{v})\in \mathcal{N}_1\times\mathcal{N}_2$, then for any $t>0$\begin{equation}
        \begin{aligned}
\label{A.1}            &\mathbbm{P}\Big(\sup_{\bm{u}\in \mathcal{N}_1}\sup_{\bm{v}\in \mathcal{N}_2}\big|\frac{1}{n}\sum_k\big(\bm{u}^\top\bm{a}_k\bm{b}_k^\top\bm{v}-\mathbbm{E}(\bm{u}^\top\bm{a}_k\bm{b}_k^\top\bm{v})\big|\geq t\Big)\\&\leq 2\exp\Big((d_1+d_2)\log 9-c_2n\min\big\{\big(\frac{t}{E_1E_2}\big)^2,\frac{t}{E_1E_2}\big\}\Big).
        \end{aligned}
    \end{equation}
    We take $t=C_3E_1E_2\sqrt{\frac{d_1+d_2}{n}}$ with sufficiently large $C_3$, recall that $n\geq d_1+d_2$, then the event \begin{equation}\begin{aligned}
        \sup_{\bm{u}\in \mathcal{N}_1}\sup_{\bm{v}\in \mathcal{N}_2}&\big|\frac{1}{n}\sum_k\big(\bm{u}^\top\bm{a}_k\bm{b}_k^\top\bm{v}-\mathbbm{E}(\bm{u}^\top\bm{a}_k\bm{b}_k^\top\bm{v})\big|\\&\leq C_3E_1E_2\sqrt{\frac{d_1+d_2}{n}}\label{A.2}\end{aligned}
    \end{equation}
    holds with probability at least $1-\exp(-c_4(d_1+d_2))$. Note that \cite[Exercise 4.4.3]{vershynin2018high} gives $\|\frac{1}{n}\sum_k\{\bm{a}_k\bm{b}_k^\top -\mathbbm{E}(\bm{a}_k\bm{b}_k^\top) \}\|_{op}\leq 2\cdot (\mathrm{the~left~hand~side~of~(\ref{A.2})})$, the proof is complete.  
\end{proof}
\subsubsection{A Lemma for the proof of Theorem \ref{theorem3}}

\begin{lemma}\label{lemma5}
    Assume $a_1,...,a_n\in \mathbb{R}$ are independent   and satisfy $\max_k \|a_k\|_{\psi_2}\leq K$; $\bm{B}_1,...,\bm{B}_n\in \mathbb{R}^{p\times q}$ are independent  and satisfy $ \sup_{\bm{u}\in \mathbb{R}^{p-1}}\sup_{\bm{v}\in \mathbb{R}^{q-1}}\|\bm{u}^\top\bm{B}_k\bm{v}\|_{\psi_2}\leq E$  for each $k$. Assume $n\gtrsim p+q$, then it holds with probability at least $1-2\exp(-c(p+q))$ that, \begin{equation}\nonumber
        \Big\|\frac{1}{n}\sum_{k=1}^n\big\{a_k\bm{B}_k - \mathbbm{E}(a_k\bm{B}_k)\big\}\Big\|_{op}\leq CKE\sqrt{\frac{p+q}{n}}.
    \end{equation}
\end{lemma}
\begin{proof}
Similarly to that of Lemma \ref{lem3}, the proof is essentially a standard covering argument for controlling operator norm of random matrix. For simplicity we assume $\mathbbm{E}(a_k\bm{B}_k)=0$; the proof extends to $\mathbbm{E}(a_k\bm{B}_k)\neq 0$ by simple centering technique \cite[Exercise 2.7.10]{vershynin2018high}.
    We invoke a covering argument: let $\mathcal{N}_1,\mathcal{N}_2$ be the $\frac{1}{4}$-net of $\mathbb{S}^{p-1}$, $\mathbb{S}^{q-1}$, respectively; we can assume $|\mathcal{N}_1|\leq 9^p,|\mathcal{N}_2|\leq 9^q$. By \cite[Exercise 4.4.3]{vershynin2018high} we have \begin{equation}
        \label{discrete}
        \Big\|\frac{1}{n}\sum_{k=1}^na_k\bm{B}_k\Big\|_{op}\leq 2\sup_{\bm{u}\in \mathcal{N}_1}\sup_{\bm{v}\in \mathcal{N}_2}\frac{1}{n}\sum_{k=1}^na_k\bm{u}^\top \bm{B}_k\bm{v}.
    \end{equation}
    For fixed $\bm{u},\bm{v}$,   $$\|a_k\bm{u}^\top\bm{B}_k\bm{v}\|_{\psi_1}\leq \|a_k\|_{\psi_2}\|\bm{u}^\top \bm{B}_k \bm{v}\|_{\psi_2}\leq KE.$$ Thus, we can apply Bernstein's inequality \cite[Theorem 2.8.1]{vershynin2018high}, together with a union bound on $\mathcal{N}_1\times \mathcal{N}_2$, to obtain   that for any $t>0$,\begin{equation}
        \begin{aligned}\nonumber
          &\mathbbm{P}\Big(\sup_{\bm{u}\in\mathcal{N}_1}\sup_{\bm{v}\in \mathcal{N}_2}\frac{1}{n}\sum_{k=1}^n a_k\bm{u}^\top\bm{B}_k\bm{v}\geq t\Big)\\&~~~\leq 2\exp\Big((p+q)\log 9 - cn\cdot\min\big\{\frac{t^2}{K^2E^2},\frac{t}{KE}\big\}\Big).
        \end{aligned}
    \end{equation}
    We set $t=C KE\sqrt{\frac{p+q}{n}}$ with sufficiently large $C$, recall that we assume $n\gtrsim p+q$, we obtain that with probability at least $1-2\exp(-c_1(p+q))$, 
    $$\sup_{\bm{u}\in\mathcal{N}_1}\sup_{\bm{v}\in \mathcal{N}_2}\frac{1}{n}\sum_{k=1}^n a_k\bm{u}^\top\bm{B}_k\bm{v} \leq CKE\sqrt{\frac{p+q}{n}}.$$
    Combined with (\ref{discrete}), the result follows.
\end{proof}
  \end{appendix}
  \begin{IEEEbiographynophoto}{Junren Chen}
is currently pursuing the Ph.D. degree with  Department of Mathematics, The University of Hong Kong. He received a Hong Kong PhD fellowship from Hong Kong Research Grants Council for supporting his Ph.D. study.  Before that, he got the B.Sc. on Mathematics and Applied Mathematics from Sun Yat-Sen University. His research interests include compressed sensing, high-dimensional statistics, signal and image processing, quantization and optimization.
\end{IEEEbiographynophoto}
\begin{IEEEbiographynophoto}{Yueqi Wang}
  received the B.S. degree from Zhejiang University, Zhejiang, China in 2021. She is currently pursuing the Ph.D. degree at  the University of Hong Kong, Hong Kong, China. Her major research interests include Photonic dispersion relation reconstruction, topological optimization, and machine learning.
\end{IEEEbiographynophoto}
  \begin{IEEEbiographynophoto}{Michael K. Ng}
(Senior Member, IEEE)  received the B.Sc. and M.Phil. degrees from The University of Hong Kong, Hong Kong, in 1990 and 1992, respectively, and the Ph.D. degree from The Chinese University of Hong Kong, Hong Kong, in 1995. From 1995 to 1997, he was a Research Fellow with the Computer Sciences Laboratory, The Australian National University, Canberra, ACT, Australia. He was an Assistant Professor/Associate Professor with The University of Hong Kong from 1997 to 2005. He was a Professor/Chair Professor (2005-2019) with the Department of Mathematics, Hong Kong Baptist University, Hong Kong, Chair Professor (2019-2023) with the Department of Mathematics, The University of Hong. He is currently a Chair Professor in Mathematics and Chair Professor in
Data Science at Hong Kong Baptist University. His research interests include 
applied and computational mathematics, machine learning and artificial intelligence, and 
data science.
Dr. Ng serves as an editorial board member of several international journals. He was selected for the 2017 Class of Fellows of the Society for Industrial and Applied Mathematics. He received the Feng Kang Prize for his significant contributions to scientific computing.
\end{IEEEbiographynophoto}
\end{document}